\documentclass{article}

\usepackage{microtype}
\usepackage{graphicx}
\usepackage{subfigure}
\usepackage{booktabs} 
\usepackage{enumitem}
\usepackage{color}
\usepackage{soul}
\usepackage{wrapfig}
\usepackage{overpic}
\usepackage{titletoc}
\usepackage[page,header]{appendix}
\usepackage[normalem]{ulem}

\usepackage{hyperref}

\usepackage[accepted]{icml2024}

\usepackage{amsmath}
\usepackage{amssymb}
\usepackage{mathtools}
\usepackage{amsthm}

\usepackage{tabularx} 
\usepackage{siunitx} 

\usepackage[capitalize,noabbrev]{cleveref}

\theoremstyle{plain}
\newtheorem{theorem}{Theorem}[section]
\newtheorem{proposition}[theorem]{Proposition}
\newtheorem{lemma}[theorem]{Lemma}

\theoremstyle{definition}
\newtheorem{definition}[theorem]{Definition}
\newtheorem{assumption}[theorem]{Assumption}
\theoremstyle{remark}
\newtheorem{remark}[theorem]{Remark}

\usepackage[textsize=tiny]{todonotes}

\usepackage{commands}

\icmltitlerunning{Probabilistic forecasting with stochastic interpolants and F\"ollmer Processes}

\begin{document}

\twocolumn[
\icmltitle{Probabilistic Forecasting with Stochastic Interpolants and F\"ollmer Processes}

\icmlsetsymbol{equal}{*}

\begin{icmlauthorlist}
\icmlauthor{Yifan Chen}{equal,nyu}
\icmlauthor{Mark Goldstein}{equal,nyu}
\icmlauthor{Mengjian Hua}{equal,nyu}\\
\icmlauthor{Michael S. Albergo}{nyu}
\icmlauthor{Nicholas M. Boffi}{nyu}\\
\icmlauthor{Eric Vanden-Eijnden}{nyu}
\end{icmlauthorlist}

\icmlaffiliation{nyu}{Courant Institute of Mathematical Sciences, New York University, New York, NY, USA}

\icmlcorrespondingauthor{Eric Vanden-Eijnden}{eve2@nyu.edu}

\icmlkeywords{Machine Learning, ICML}

\vskip 0.3in
]

\printAffiliationsAndNotice{\icmlEqualContribution} 

\begin{abstract}
We propose a framework for probabilistic forecasting of dynamical systems based on generative modeling.
Given observations of the system state over time, we formulate the forecasting problem as sampling from the conditional distribution of the future system state given its current state.
To this end, we leverage the framework of stochastic interpolants, which facilitates the construction of a generative model between an arbitrary base distribution and the target. 
We design a fictitious, non-physical stochastic dynamics that takes as initial condition the current system state and produces as output a sample from the target conditional distribution in finite time and without bias. 
This process therefore maps a point mass centered at the current state onto a probabilistic ensemble of forecasts. 
We prove that the drift coefficient entering the stochastic differential equation (SDE) achieving this task is non-singular, and that it can be learned efficiently by square loss regression over the time-series data.
We show that the drift and the diffusion coefficients of this SDE can be adjusted after training, and that a specific choice that minimizes the impact of the estimation error gives a F\"ollmer process.
We highlight the utility of our approach on several complex, high-dimensional forecasting problems, including stochastically forced Navier-Stokes and video prediction on the KTH and CLEVRER datasets.
The code is available 
at 
\href{https://github.com/interpolants/forecasting}{https://github.com/interpolants/forecasting}.
\end{abstract}

\section{Introduction}
\label{sec:intro}
Forecasting the future state of a dynamical system given complete or partial information about the current state is a ubiquitous problem across science and engineering, with diverse applications in climate modeling~\cite{smagorinsky1963, palmer1992, Gneiting2005, pathak2022fourcastnet}, fluid dynamics~\cite{buaria_forecasting_2023}, video prediction~\cite{oprea_review_2022, finn_unsupervised_2016, lee_stochastic_2018}, and extrapolation of time series data~\cite{lim_time-series_2021, masini_machine_2023}.
Forecasting can be performed \textit{deterministically}, whereby the goal is to compute a single prediction for the future state~\cite{giannakis_learning_2023}, or \textit{probabilistically}, whereby the goal is to predict a distribution over future states consistent with the current information~\cite{gneiting2014}.
Probabilistic forecasting is the natural formulation when the underlying system dynamics are stochastic, when the full system state cannot be measured, or when measurements are corrupted by noise, as is the case for most real-world systems of interest.
Moreover, while deterministic forecasting appears to be a simpler problem, recent works have uncovered difficulties with deterministic forecasting methods when the underlying dynamics are chaotic~\cite{jiang2023training}.

Motivated by the recent success of generative models built upon dynamical transport of measure, such as score-based diffusion \cite{ho2020denoising, song2020score}, flow matching \cite{lipman2022flow}, and stochastic interpolants \cite{albergo2022building, albergo2023stochastic}, here we introduce a generative modeling approach for probabilistic forecasting. 
This approach maps the current state of the system onto the ensemble of possible future outcomes after a fixed time lag. From a transport perspective, this requires pushing a point mass measure onto a probability distribution with larger support.
In what follows, we show that the interpolant framework enables us to design an artificial dynamics that performs this task using a stochastic differential equation. 
In practice, the drift fields entering the SDEs we introduce can be learned via square loss regression.
We also show that the diffusion coefficient in these SDEs can be tuned \textit{a-posteriori} (i.e. without having to retrain a drift).
We show that a specific choice that minimizes the impact of the estimation error recovers a F\"ollmer process~\cite{follmer1986time}, a specific instantiation of the Schr\"odinger bridge problem~\cite{schrodinger1932theorie,leonard2014survey} in which the base distribution is a point mass measure. 

To demonstrate the utility and scalability of our approach, we consider several examples: an illustrative low-dimensional scenario based on a multi-modal jump diffusion process, a high-dimensional problem generated by a stochastic Navier-Stokes equation on the torus, and video generation on the KTH \cite{schuldt2004recognizing} and CLEVRER datasets \cite{yi2019clevrer}.
In the case of Navier-Stokes, we show that our probabilistic forecasting method is capable of reproducing quantitative metrics, such as the enstrophy spectrum of the dataset, using either high- or low-resolution measurement of the current system state.
In each case, we highlight that we are able to obtain diverse samples consistent with the conditional distribution of interest, and demonstrate the need for probabilistic, as opposed to deterministic, forecasting. 
For the video generation tasks, we show that our learned models are more effective than standard conditional generative modeling.
We also show that this forecasting procedure can be iterated autoregressively without retraining to compute a predicted trajectory.
Our \textbf{main contributions} can be summarized as follows:
\begin{itemize}[leftmargin=0.15in]
    \item We design new generative models for probabilistic forecasting based on stochastic differential equations (SDEs) that map a point mass measure to a distribution with full support by incorporating stochasticity in a principled way, enabling us to initialize the SDE directly at the measured system state.
    \item We prove that the drifts entering these SDEs can be learned via square loss regression over the data, and that the resulting loss has bounded variance.
    \item We show that the drift and noise terms in these SDEs can be adjusted post-training, and that the specific choice of noise that minimizes the Kullback-Leibler (KL) divergence between the path measures of the exact forecasting process and the estimated one is realizable. We show that this drift and diffusion pair gives a F\"ollmer process.
    \item We validate our theoretical results empirically on several challenging high-dimensional forecasting tasks, including the Navier-Stokes equation and video prediction.
\end{itemize}

\section{Related Work}
\label{sec:related}
There is a vast body of literature on forecasting. Methodologies can be broadly classified into two main categories: forecasting a single output, typically achieved through regression and operator theoretic approaches \cite{kutz2016dynamic,alexander2020operator,li2021learning}, and probabilistic forecasting \cite{gneiting2014}, based on stochastic and generative modeling. 

\vspace{-0.5em}
\paragraph{\textbf{Deterministic vs Probabilistic Forecasting.}}
For the class of methodologies focused on generating a single output,
a widely used approach is regression or supervised learning, which entails directly learning a map in the state space. 
One can also work in the space of probability densities or functionals on the state space, under the setting of the Frobenius-Perron or Koopman operator approach \cite{dellnitz1999approximation,kaiser2021data}; this leads to linear dynamics in an infinite-dimensional space, and forecasting reduces to finding a tractable finite-dimensional approximation.
In all formulations, the key is to identify and learn accurate representations of the dynamics, for example through nonparametric approaches such as diffusion maps \cite{berry2015nonparametric} and kernel regression \cite{alexander2020operator}, or parametric approaches such as linear regression, dynamical mode decomposition \cite{kutz2016dynamic}, neural networks \cite{li2021learning, gu2021efficiently} and operators \cite{lu2021learning,jiang2023training,li2020fourier}. Ultimately, these approaches produce a function that maps the current state to a single output which is the deterministic forecast. 
Many methods under this category train models with the MSE (mean square error) or RMSE (root mean square error) as the objective, but these losses may be poor signals for training forecasters for chaotic systems \cite{jiang2023training}. 

For dynamics that are inherently stochastic, or with incomplete information, a stochastic forecast is important to incorporate uncertainties. This stochasticity can be introduced by fitting a probabilistic model, such as a stochastic process or a graphical model, to data. They can also be approached by stochastic Koopman operators \cite{wanner2022robust,zhao2023data}.
Although many approaches have been proposed, most of them target low-dimensional problems or simple conditional statistics such as Gaussians. 

\vspace{-1em}
\paragraph{\textbf{Probabilistic Forecasting with Generative Models.}}
Recently, generative modeling techniques in machine learning have received increasing attention for handling high-dimensional, complex distributions. Probabilistic forecasting can be seen as a form of the conditional generation problem. 
Various conditional generative models, including conditional GAN \cite{mirza2014conditional}, VAE \cite{sohn2015learning}, and normalizing flows \cite{Kalman_Normalizing, kidger2021neural} have been developed. More recently, diffusion generative models \cite{ho2020denoising,song2020score} have gained popularity due to their state-of-the-art performance. The stochastic interpolant methodology, which is related to contemporary work like flow matching \cite{lipman2022flow}, is a general framework that encompasses diffusion models. 
In the literature, there has been some work to develop conditional models for time series \cite{rasul2021autoregressive, Lienen2023FromZT},
including diffusion 
\cite{ho2022imagen,blattmann2023stable}
and flow matching models for video prediction \cite{davtyan2023efficient}.
These approaches learn ODEs or SDEs that map a Gaussian base to the conditional distribution of interest. 
There are also works that forecast stochastic dynamics by adding noise to the neural network layers \cite{cachay2023dyffusion} using techniques such as dropout. 
In contrast, we incorporate stochasticity rigorously through interpolants and direct SDE modeling. 
Our construction of the stochastic generative model, which maps the current state to the distribution of the forecasted state, is new and can be seen as a direct stochastic extension of the deterministic map approach that is predominant in single output forecasting. 

\vspace{-0.5em}
\paragraph{\textbf{F\"ollmer processes.}} {The construction of SDEs that map a point mass to a target distribution dates back at least to the F\"ollmer process \cite{follmer1986time}, which is a particular solution of the Schr\"odinger bridge problem \cite{schrodinger1932theorie,leonard2014survey,chen2021stochastic} that minimizes the relative entropy with respect to the Wiener process. This approach has the desirable feature that it offers an entropy-regularized solution to the optimal transport problem.
For this reason, the concepts of the F\"ollmer process and the Schr\"odinger bridge have found many applications in sampling densities with unknown normalization constants \cite{zhang2021path,huang2021schrodinger, jiao2021convergence,vargas2023bayesian}, generative modeling \cite{tzen2019theoretical, wang2021deep,debortoli2021diffusion,liu20232,peluchetti2023non,shi2024diffusion} as well as stochastic analysis and functional inequalities \cite{lehec2013representation,eldan2018regularization,eldan2020stability}.
In this article, we show that the  stochastic interpolant framework offers a simple way to construct new types of F\"ollmer processes.
In addition, we give a new interpretation of F\"ollmer processes as minimizers of a KL divergence between the path measure of an SDE that forecasts exactly and the approximate, learned SDE. 
These results support the uses of F\"ollmer processes in probabilistic forecasting.}

\section{Setup and Main Results}
\label{sec:setup}

\subsection{Conditional PDF}
\label{sec:cond}
Assume that we are given a joint probability density function (PDF) $\rho(x_0,x_1)$ supported on $\R^d \times \R^d$ and strictly positive everywhere.
Our aim is to design a generative model to sample the conditional PDF of $x_1$ given $x_0$:
\begin{equation}
    \label{eq:pdf:c}
    \rho_c(x_1|x_0) = \frac{\rho(x_0,x_1)}{\rho_0(x_0)}>0, 
\end{equation}
where $\rho_0(x_0) = \int_{\R^d} \rho(x_0,x_1) dx_1>0$.\footnote{%
Our approach actually requires that $x_1|x_0$ has a positive density, but $(x_0,x_1)$ does not have to. In particular, we can target any density $\rho_1(x_1)>0$ by drawing $(x_0,x_1)$ from the joint distribution $\mu(dx_0,dx_1) = \mu_0(dx_0) \rho_1(x_1) dx_1$ with any $\mu_0$ (e.g. a point mass), since we then have $\rho_c(x_1|x_0)= \rho_1(x_1)$.}
Phrasing the question this general way allows us to consider several instantiations of interest by appropriately defining the joint PDF $\rho$.  In particular, we study the following three problem settings:

\paragraph{\textbf{Probabilistic forecasting.}} Suppose that we are given a discrete time-series $\{\ldots, x_{-\tau},x_0,x_\tau,\ldots\}= \{x_{k\tau}\}_{k\in \Z}$ with each $x_{k\tau}\in \R^d$  containing, for example, daily weather measurements or video frames, acquired every lag-time~$\tau>0$. Assume that this time-series is a stationary process\footnote{{This assumption can be relaxed; see Appendix \ref{appendix-discussion-stationarity}.}} and that the law of successive observations $(x_{k\tau},x_{(k+1)\tau})$ is captured by the joint PDF $\rho(x_{k\tau},x_{(k+1)\tau})$.
Then sampling from $\rho_c(\cdot|x_{k\tau})$ produces the ensemble of forecasts $x_{(k+1)\tau}$ given the observation~$x_{k\tau}$. 

\vspace{-0.5em}
\paragraph{\textbf{Signal recovery from corrupted data.}} Suppose that, given clean data $x_1\in \R^d$, we observe the corrupted signal $x_0 \in \R^d$ (e.g. a low-resolution or noisy image obtained from a high-resolution image). %
If we assume that the joint PDF of $(x_0,x_1)$ is $\rho(x_0,x_1)$, then sampling from $\rho_c(\cdot|x_0)$  produces the ensemble of clean data $x_1$ consistent with the corrupted signal $x_0$. 

\vspace{-0.5em}
\paragraph{\textbf{Probabilistic forecasting from noisy observations.}} We can combine the previous two setups if we assume that we are given a discrete-time-series $\{x_{k\tau},\tilde x_{k\tau}\}_{k\in \Z}$  with $x_{k\tau}$ the clean data and $\tilde x_{k\tau}$ the corrupted observation at time $k\tau$. If we assume that the joint PDF of $(\tilde x_{k\tau}, x_{(k+1)\tau})$ is $\rho(\tilde x_{k\tau}, x_{(k+1)\tau})$, then sampling from $\rho_c(\cdot|\tilde x_{k\tau})$ produces the ensemble of clean forecasts $x_{(k+1)\tau}$ given the noisy observation $\tilde x_{k\tau}$. 

\subsection{Generation with Stochastic Interpolants}
\label{sec:stoch:interp}

The generative models that we develop here are based on stochastic differential equations that map a fixed initial condition $X_{s=0} = x_0$ to samples from the conditional distribution $X_{s=1}\sim \rho_c(\cdot|x_0)$.
Towards the design of such SDEs, we first introduce the \textit{stochastic interpolant}
\begin{equation}
    \label{eq:stoch:int:def}
    I_s = \alpha_s  x_0 + \beta_s  x_1 + \sigma_s  W_s
\end{equation} 
where $(x_0,x_1)\sim \rho(x_0,x_1)$ and $W= (W_s)_{s\in[0,1]}$ is a Wiener process with $W\perp(x_0,x_1)$. %
In addition, we impose that $\alpha,\beta,\sigma\in C^1([0,1])$ satisfy the boundary conditions $\alpha_0=\beta_1=1$ and $\alpha_1=\beta_0=\sigma_1=0$.
To facilitate some calculations, we assume that $\dot\beta_s >0$ for all $s\in(0,1]$ and $\dot \sigma_s<0$ for all $s\in[0,1]$.
Here we will use $\alpha_s=\sigma_s = 1-s$, and $\beta_s=s$ or $\beta_s=s^2$ (see Appendix~\ref{app:specifics}).
This second choice for $\beta_s$ has some advantages that we discuss below. 

The boundary conditions on $\alpha,\beta,$ and $\sigma$ guarantee that $I_{s=0} = x_0$ and $I_{s=1} = x_1$, so that the probability distribution of $I_s|x_0$ bridges the point mass measure at $x_0$ to $\rho_c(\cdot|x_0)$ as $s$ varies from 0 to 1.
The following result shows that this probability distribution is also the law of the solution to a SDE that can be used as a generative model.
\begin{theorem}
\label{th:2:gen}
Let $b_s(x,x_0)$ be the unique minimizer over all $\hat b_s(x,x_0)$ of the objective
\begin{equation}
\label{eq:loss}
    L_{b}[\hat b_s] = \int_0^1 \E\big [ |\hat b_s(I_s ,x_0)  - R_s |^2] ds,
\end{equation}
where $\E$ denotes an expectation over $(x_0,x_1)\sim \rho$ and $W$ with $(x_0,x_1) \perp W$, $I_s$ is given in~\eqref{eq:stoch:int:def}, and we defined\footnote{%
Here and below the dot denotes derivative with respect to~$s$.}
\begin{equation}
    \label{eq:x:r}
    \begin{aligned}
     R_s  = \dot\alpha_s  x_0 + \dot\beta_s  x_1 +  \dot\sigma_s W_s,
     \end{aligned}
\end{equation}
Then the solutions to the SDE
\begin{equation}
    \label{eq:sde}
    dX_s= b_s(X_s,x_0) ds + \sigma_s  dW_s, \quad X_{s=0} = x_0,
\end{equation}
are such that $\text{Law}(X_s) = \text{Law}(I_s|x_0) $ for all $(s,x_0)\in [0,1]\times \R^d$.
In particular, $X_{s=1}\sim \rho_c(\cdot |x_0)$.
\end{theorem}
This theorem is proven in Appendix~\ref{app:proof:th1}. 
The result is formulated in a way that is tailored to practical approximation of the drift~$b_s$, since the objective~\eqref{eq:loss} can be estimated empirically by generating samples of $I_s$ and $R_s$ using sample pairs $(x_0,x_1)$ from $\rho$ and realizations of $W_s \stackrel{d}{=} \sqrt{s} z$ with $z\sim {\sf N}(0,\Id)$. 
That is, $b_s$ may be learned over neural networks by minimizing the simulation-free loss~\eqref{eq:loss} over the parameters.
It is easy to see that the minimizer is given by
\begin{equation}
    \label{eq:drift:cond}
    b_s(x,x_0) = \E^{x_0}[R_s | I_s = x],
\end{equation}
where $\E^{x_0}[\cdot|I_s =x]$ denotes an expectation over $x_1\sim \rho_c(\cdot|x_0)$ and $W$ with $x_1\perp W$ conditioned on the event $I_s =x$. 
The drift~\eqref{eq:drift:cond} is well-defined for all $(s,x,x_0)\in [0,1]\times\R^d\times \R^d$, and we show in Appendix~\ref{app:reg} that enforcing $\dot \beta_0 = 0$ offers additional control on the boundedness and Lipschitz constant of~$b$ at the initial time.
We observe empirically that this has computational advantages at both optimization and sampling time; see the experiments in Appendix~\ref{sec:detail:num}. 

\subsection{Generalizations with Tunable Diffusion}
\label{sec:gen:tunable}

We now show that learning the drift coefficient \eqref{eq:drift:cond} gives access to a broader set of SDEs to use as a generative model beyond just \eqref{eq:sde}, and that selecting from them has appealing theoretical motivation.

Let $\rho_s(x|x_0)$ be the PDF of $X_s \stackrel{d}{=} I_s|x_0$.
From~\eqref{eq:sde}, $\rho_s$ solves the Fokker-Planck equation
\begin{equation}
    \label{eq:FPE}
    \partial_s \rho_s + \nabla \cdot (b_s(x,x_0) \rho_s) = \tfrac12 \sigma^2_s \Delta \rho_s.
\end{equation}
Given a candidate diffusion coefficient $g_s$, we can use the identity $\tfrac12 \sigma^2_s \Delta \rho_s = \tfrac12 g^2_s \Delta \rho_s -\tfrac12 (g_s^2-\sigma_s^2) \nabla \cdot (\rho_s \nabla \log \rho_s)$ to trade diffusion for transport in~\eqref{eq:FPE}.
This construction leads to a family of SDEs with tunable diffusion 
\begin{theorem}
\label{th:1:gen}
Given any $g\in C^0([0,1])$ such that $\lim_{s\to0^+} s^{-1}[g_s^2-\sigma^2_s]$ and $\lim_{s\to1^-} g^2_s \sigma^{-1}_s$ exist, define
\begin{equation}
\label{eq:drift:bc}
    b^g_s(x,x_0) = b_s(x,x_0) + \tfrac12 (g^2_s -\sigma_s^2) \nabla \log \rho_s(x|x_0)
 \end{equation}
where  $b_s(x,x_0)$ is the  minimizer of~\eqref{eq:loss} given in~\eqref{eq:drift:cond} and $\rho_s(x|x_0)$ is the PDF of $X_s \stackrel{d}{=} I_s|x_0$.
Then the solutions to the SDE
\begin{equation}
    \label{eq:sde:g}
    dX^g_s= b^g_s(X^g_s,x_0) ds + g_s  dW_s, \quad X^g_{s=0} = x_0,
\end{equation}
are  such that $\text{Law}(X^g_s)= \text{Law}(X_s) = \text{Law}(I_s|x_0) $ for all $(s,x_0)\in [0,1]\times \R^d$. In particular $X^g_{s=1}\sim \rho_c(\cdot |x_0)$.
\end{theorem}
This theorem is proven in Appendix~\ref{app:tunable}, where we explain why the conditions on $g_s$ guarantee that the SDE~\eqref{eq:sde:g} is well-posed.
Working with this SDE requires the score $\nabla \log \rho_s$. Interestingly, this score can be expressed in terms of the drift~$b_s$.
A direct calculation reported in Appendix~\ref{app:tunable} shows that 
\begin{equation}
    \label{eq:score}
    \nabla \log \rho_s(x|x_0) = A_s \left[\beta_s b_s(x,x_0) - c_s(x,x_0) \right],
\end{equation}
where 
\begin{equation}
    \label{eq:df:A:c}
    \begin{aligned}
    A_s &= [s\sigma_s (\dot\beta_s \sigma_s -\beta_s  \dot\sigma_s )]^{-1},\\
    c_s(x,x_0) &= \dot \beta_s  x + (\beta_s  \dot \alpha_s  - \dot\beta_s  \alpha_s )x_0.
    \end{aligned}
\end{equation} 
Using~\eqref{eq:score} in~\eqref{eq:drift:bc} shows that, to work with the SDE~\eqref{eq:sde:g}, we can estimate~$b$ first and then adjust both the noise amplitude $g_s$ and the drift~$b^g$ \textit{a-posteriori}  without having to retrain~$b$.\footnote{%
Using~\eqref{eq:score} in~\eqref{eq:drift:bc} requires some care at $s=0$ and $s=1$ due to the factor $[s\sigma_s]^{-1}$ in $A_s$, but this leads to no issue see Algorithm~\ref{alg:sampling} and also Appendices~\ref{app:specifics} and~\ref{app:tunable}.} 
This offers flexibility at sampling time that can be leveraged to maximize performance, as shown in our numerical experiments below.  

\subsection{KL Optimization and F\"ollmer Processes}
\label{sec:gen}

In light of Theorem~\ref{th:1:gen}, it is natural to ask if a specific choice of $g_s$ is optimal in a suitable sense.
To provide one answer to this question, we consider the KL divergence between the path measure of the process $X^g = (X^g_s)_{s\in[0,1]}$ (which solves the ideal SDE~\eqref{eq:sde:g}) and the path measure of the process $\hat X^g = (\hat X^g_s)_{s\in[0,1]}$ (which solves an approximate, learned version of~\eqref{eq:sde:g} obtained through an estimate $\hat{b}$ of $b$).
Because $\text{Law}(X^g_s) = \text{Law}(I_s|x_0) $ for all $(s,x_0)\in [0,1]\times \R^d$, this KL divergence is given by (see Appendix~\ref{app:KL} for details)
\begin{equation}
    \label{eq:KL:path}
        D_{\text{KL}} (X^g||\hat X^g)  = \int_0^1 \frac{ |1+\frac12 \beta_s  A_s(g_s^2-\sigma_s^2)|^2L_s}{2|g_s |^2} ds
\end{equation}
where 
$L_s=\E^{x_0} \big[ |\hat b_s(I_s,x_0)- b_s(I_s,x_0)|^2 \big]$. 
Eq.~\eqref{eq:KL:path} measures how the estimation error on~$b$ impacts the generative process, and as such it is natural to minimize it over~$g$. 
Since $L_s$ is independent of~$g_s$, this minimization can be performed analytically. 
The result is that~\eqref{eq:KL:path} is minimized if we set $g_s=g^\Fo_s$ with 
\begin{equation}
    \label{eq:g:spec}
    g^\Fo_s = \left| 2s\sigma_s(\beta_s^{-1}\dot\beta_s \sigma_s- \dot\sigma_s) -\sigma_s^2\right|^{1/2}.
\end{equation}
This expression is well-defined for all $s\in[0,1]$, since $\lim_{s\to0^+} 2s\beta_s^{-1}\dot\beta_s<\infty$ because $\beta_s$ is differentiable at $s=0$ by assumption.
The result in~\eqref{eq:g:spec} is also amenable to an interesting interpretation:
\begin{theorem}
    \label{th:interp:f:si}
     If $\beta_s /[\sqrt{s}\sigma_s ]$ is non-decreasing, then the process $X^\Fo \equiv X^{g^\Fo}$ that solves~\eqref{eq:sde:g} with $g_s = g^\Fo_s$ is a F\"ollmer process.
\end{theorem}
This theorem is proven in Appendix~\ref{app:foll}. To understand its significance, recall that the F\"ollmer process is the solution to the Schr\"odinger bridge problem when one of the endpoint measures is a point mass (in this case, at $x_0$).
As such, it offers an entropy-regularized solution to the optimal transport problem.
The F\"ollmer process is usually defined by minimizing its KL divergence with respect to the Wiener process subject to constraints on the endpoints.
Theorem~\ref{th:interp:f:si} offers a generalization and new interpretation of this process as the minimizer of the KL divergence of the exact forecasting process from the estimated one, which is more tailored to statistical inference.
For more details about F\"ollmer processes and the Schr\"odinger bridge problem we refer the reader to Appendix~\ref{app:KL}. 
We also test the performance of~\eqref{eq:g:spec} in Appendix~\ref{sec:detail:nse}.

\subsection{Implementation}
\label{sec:implem}
For concreteness, we consider the problem of probabilistic forecasting, but the alternative problem settings covered by our framework can be handled similarly. 
Given the truncated time series $\mathcal{S}_K = \{x_{k\tau}\}_{k=0}^{K+1}$ with $K\in \N$ and $K'\le K$, we can approximate the objective in~\eqref{eq:loss} by the empirical loss 
\begin{equation}
\label{eq:loss:emp}
\begin{aligned}
    L^K_b[\hat b] = &\frac1{K'} \sum_{k\in B_{K'}}\int_0^1  |\hat b_s(I^k_s,x_{k\tau})-R^k_s |^2 ds,
\end{aligned}
\end{equation}
where $B_{K'}\subset \{0:K\}$ is a subset of indices of cardinality~$K'$ and
\begin{equation}
    \label{eq:IRsk}
    \begin{aligned}
    I_s^k &= \alpha_s  x_{k\tau} + \beta_s  x_{(k+1)\tau} + \sqrt{s} \sigma_s  z_k\\ 
    R_s^k &= \dot\alpha_s  x_{k\tau} + \dot\beta_s  x_{(k+1)\tau} + \sqrt{s} \dot\sigma_s  z_k
    \end{aligned}
\end{equation}
with $z_k \sim {\sf N}(0,Id)$, $z_k\perp(x_{k\tau},x_{(k+1)\tau})$. To arrive at~\eqref{eq:loss:emp} we used that $W_s \stackrel{d}{=} \sqrt{s} z$ with $z\sim {\sf N}(0,Id)$ at all $s\in[0,1]$. 
In~\eqref{eq:loss:emp} and~\eqref{eq:sde:emp} below, the physical lag $\tau > 0$ is fixed, while~$s$ varies over $[0,1]$, and the integral over~$s$ can be approximated via an empirical expectation over draws of $s \sim {\sf U}([0,1])$.
By approximating $\hat b$ in an expressive parametric class such as a class of neural networks, we can optimize~\eqref{eq:loss:emp} over the parameters with standard gradient-based methods.
This can be performed by batching over subsequences in the available time series, or via online learning if a stream of data is continuously observed.

Having learned an approximation $\hat b$, we can construct an approximation of $\hat b^g$ using~\eqref{eq:drift:bc},~\eqref{eq:score}, and~\eqref{eq:df:A:c}.
We may then form our model given the new observation $x_{k\tau}$ by solving
\begin{equation}
    \label{eq:sde:emp}
    d\hat{X}^k_s = \hat b^g_s(\hat{X}^k_s,x_{k\tau}) ds + g_s  dW_s, \quad \hat{X}^k_{s=0} = x_{k\tau},
\end{equation}
with various realizations of the noise $W_s$ to generate a set of $\hat{X}^k_{s=1}$ that approximately samples $\rho_c(\cdot|x_{k\tau})$.
This generates an ensemble of forecasts with statistics consistent with those of the time-series seen during training. 
This process can also be iterated autoregressively by setting $\hat{X}^{k+1}_{s=0} = \hat{X}^k_{s=1}$ and by solving the SDE~\eqref{eq:sde:emp} with $k$ replaced by $k+1$ to get an approximate sample $\hat{X}^{k+1}_{s=1}$ of $\rho_c(\cdot|x_{(k+1)\tau})$.
This iteration does not require any additional training, since it uses the same $\hat{b}_s$.
These procedures are summarized in Algorithms~\ref{alg:training} and \ref{alg:sampling}. The first step to get $X_{s=s_1} = \hat X_1$ in Algorithm~\ref{alg:sampling} is consistent and designed so that it avoids computing $\hat b^g_{s=0}$, since using~\eqref{eq:drift:bc} and~\eqref{eq:score} can exhibit numerical singularities even though $b^g_{s=0}$ is well defined when the conditions of Theorem~\ref{th:1:gen} are met. 


%
\begin{algorithm}
\caption{Training}
\label{alg:training}
\begin{algorithmic}[1]
\STATE \textbf{Input}:  Data set $\mathcal {S}_K= \{x_{k\tau}\}_{k=0}^{K+1}$; minibatch size $K'\le K$; coefficients $\alpha_s,\beta_s,\sigma_s$.
\REPEAT
\STATE Compute the  empirical loss $L^K_b[\hat b]$ in~\eqref{eq:loss:emp}.
\STATE Take gradient step on $L^K_b[\hat b]$ to update $\hat b_s$.
\UNTIL{converged}
\STATE \textbf{Return}: drift $\hat b_s(x,x_0)$.
\end{algorithmic}
\end{algorithm}

\vspace{-0.5em}
\begin{algorithm}
\caption{Sampling}
\label{alg:sampling}
\begin{algorithmic}[1]
\STATE \textbf{Input}:  Observation $x_{k\tau}$; model $\hat b_s(x,x_0)$; noise coefficient $g_s$, grid $s_0=0<s_1 \cdots< s_{N} =1$ with $N\in \N$; \textit{iid} $\eta_n \sim {\sf N}(0,\Id)$ for $n=0:N-1$.
\STATE Set $\Delta s_n = s_{n+1}-s_n$, $n=0:N-1$.
\STATE Set $\hat X_1 = x_{k\tau} +  \hat b_{s_0}(x_{k\tau},x_{k\tau}) \Delta s_0 + \sigma_{s_0} \sqrt{\Delta s_0}  \eta_0$. 
%
%
\FOR{$n = 1:N-1$}
\STATE Compute $\hat b^g_{s_n}(\hat X_n, x_{k\tau})$ from~\eqref{eq:drift:bc} and~\eqref{eq:score}.
\STATE Set $\hat X_{n+1} = \hat X_{n} +  \hat b_{s_n}^g(\hat X_n, x_{k\tau}) \Delta s_n + g_{s_n} \sqrt{\Delta s_n}\eta_n$.
\ENDFOR
\STATE \textbf{Return}: $\hat X_{N+1} \sim \hat \rho_c(\cdot| x_{k\tau})\approx \rho_c(\cdot| x_{k\tau})$.
\end{algorithmic}
\end{algorithm}
\vspace{-1em}

\section{Numerical illustrations}
\label{sec:num}
In what follows, we test our proposed method in several application domains.
For all tests, an interpolant with coefficients $\alpha_s = 1-s$, $\sigma_s = \varepsilon(1-s)$ for some $\varepsilon>0$, and $\beta_s = s^2$ is used.
The condition that $\dot \beta_0 = 0$ empirically ensures that the norm of the parameter gradients used to train our neural networks are well behaved.
We report results with the diffusion coefficient in~\eqref{eq:sde:emp} chosen to be $g_s = \sigma_s$, as we found that the impact of learning $\hat b_s$ well by choice of the right interpolant outweighed the effect of varying the SDE for the systems we study.
For additional numerical experiments with $\alpha_s = 1-s, \sigma_s = \varepsilon(1-s)$ and $\beta_s = s$, and with $g_s=g_s^\Fo$, we refer the reader to Appendix~\ref{sec:detail:nse}.
Investigation of the F\"ollmer SDE described in~\Cref{th:interp:f:si}, in both theory and experiment, will be saved for future work.
\subsection{Multi-modal jump diffusion process}
\label{sec:gmm}

\begin{wrapfigure}[11]{r}{0.21\textwidth}
\vspace{-1em}
  \centering
  \includegraphics[width=0.19\textwidth]{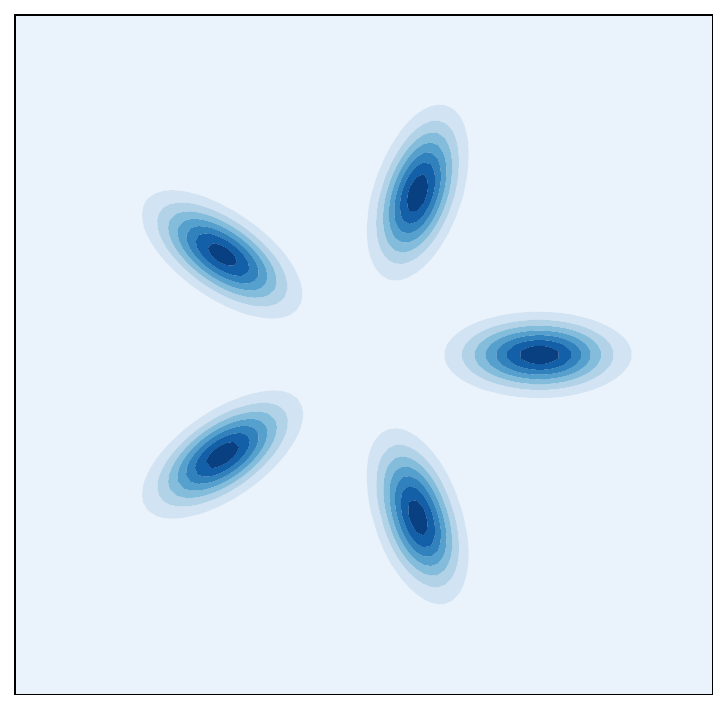}
\caption{Invariant PDF for the jump-diffusion process.}
\label{fig:gmm}
\end{wrapfigure}
Our first example is synthetic, and consists of forecasting a two-dimensional jump-diffusion process with invariant PDF given by a Gaussian mixture (Fig.~\ref{fig:gmm}).
We study a particle governed by Langevin dynamics that is randomly kicked in the counterclockwise direction, where the times between kicks are specified by a Poisson process (see Appendix~\ref{sec:detail:jump:diff} for details).

In this example, if the process starts at a point in one mode, its PDF spreads in the mode and leaks into the other nearby modes in the counterclockwise direction (see Fig.~\ref{fig-gmm-cond-stats}).
Correspondingly, the conditional PDF of $x_{\tau}$ given $x_{0}$ is itself a Gaussian mixture PDF that is sharply peaked around $x_{0}$ when $\tau$ is small, and which slowly evolves into the 5-mode invariant measure associated with the potential as $\tau$ increases. We generate a long time series of this process and use it at different lags $\tau$ in the empirical loss~\eqref{eq:loss:emp} to learn the drift velocity $\hat b$, which we model as a fully connected neural network. We then use the estimated $\hat b^g$ in the SDE~\eqref{eq:sde:emp} to generate probabilistic forecasts. 
The results (Fig.~\ref{fig-gmm-cond-stats}) indicate that the law of these forecasts is in excellent agreement with the true $\rho_c(x_{\tau}|x_0)$. %
We can also iterate using the procedure described in Sec.~\ref{sec:implem}  to estimate $\rho_c(x_{k\tau}|x_0)$ for $k > 1$ without additional retraining.
We find excellent agreement in doing so, including beyond the decorrelation time, when the conditional PDF relaxes into the equilibrium distribution independent of $x_0$.
\begin{figure}[ht!]
    \centering
    \includegraphics[width=\linewidth]{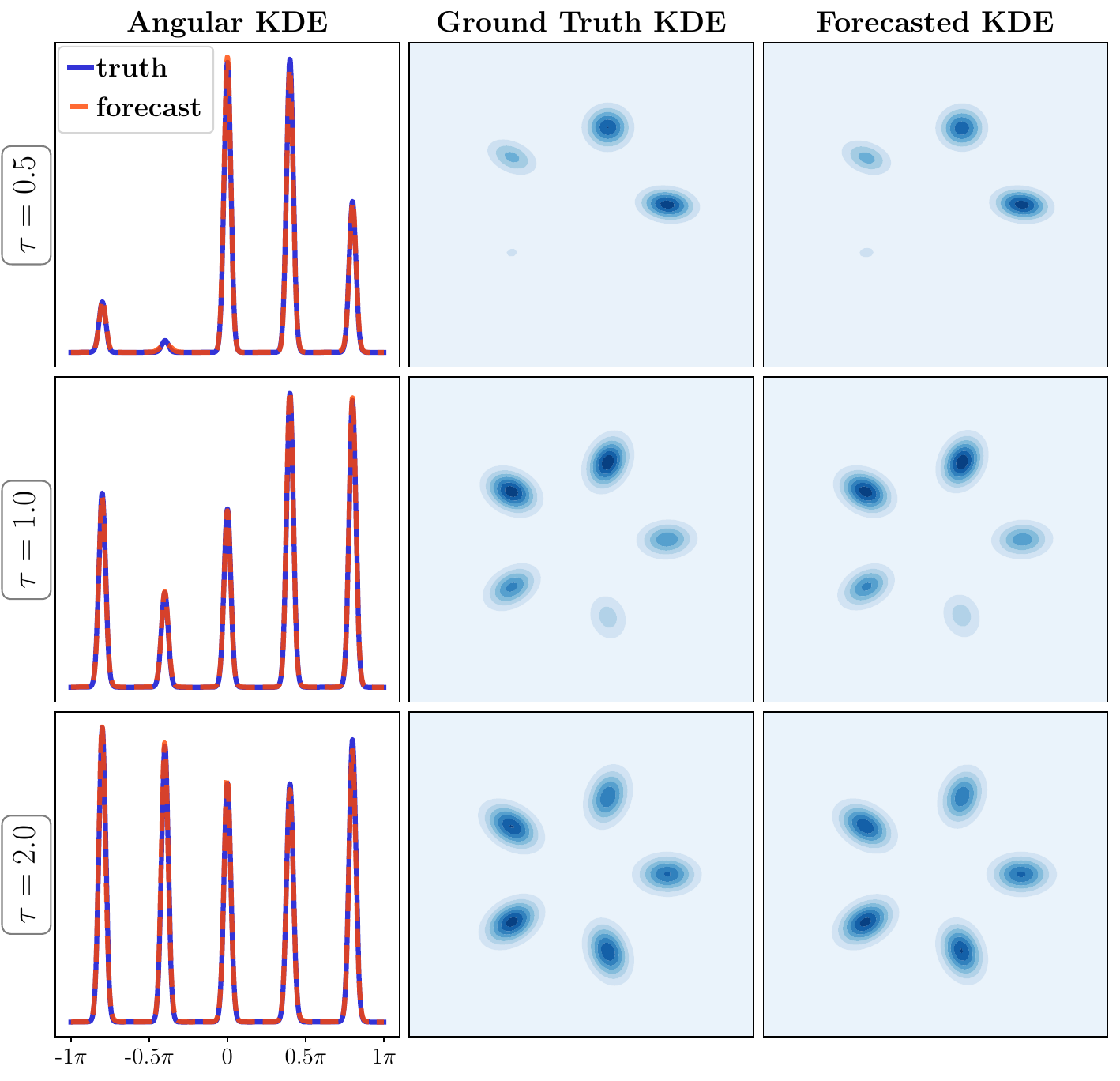}
    \caption{\textbf{Forecasting comparison for the jump-diffusion process.} 
    (Left) Comparison of the truth to the forecasted prediction in the angular coordinates.
    (Middle) Ground truth KDEs at various lag times $\tau$. 
    (Right) Forecasted KDEs at the same lag times $\tau$.}
    \label{fig-gmm-cond-stats}
\end{figure}
We note that this is an example in which probabilistic forecasting is key, as regressing $x_{\tau}$ given $x_0$ would give results with little information. 
Probabilistic forecasting is needed to capture the non-Gaussian and multimodal nature of the forecasts, which would be hard to capture with deterministic methods. 

\subsection{Forecasting the 2d Navier-Stokes Equations}
\label{sec:nse}
Our second numerical example considers forecasting the dynamics of the solution to the 2d Navier-Stokes equations with random forcing on the torus $\mathbb{T}^2 = [0,2\pi]^2$. With the vorticity formulation, the Navier-Stokes equations read
\vspace{-0.32em}
\begin{equation}
    \label{eq:2D_vorticity_NS}
    d \omega + v \cdot \nabla  \omega dt = \nu \Delta \omega dt  - \alpha \omega dt + \varepsilon d\eta.
\end{equation}
\vspace{-0.2em}
Here $v = \nabla^{\perp} \psi = (-\partial_y\psi, \partial_x\psi)$ is the velocity expressed in terms of the stream function $\psi$, which is a solution to $-\Delta \psi = \omega$, $d\eta$ is white-in-time random forcing acting on a few Fourier modes, and $\nu,\alpha,\eps>0$ are parameters  (see Appendix~\ref{sec:detail:nse} for details). We work in a setting where \eqref{eq:2D_vorticity_NS} is provably ergodic with a unique invariant measure~\cite{hairer2006ergodicity}.  Our objective is to forecast the solution to \eqref{eq:2D_vorticity_NS} at time $t+\tau$ given its solution at time $t$ after the process has reached a statistically steady state.  We do so using both full- and low-resolution data at time $t$, but our goal is always to forecast at full resolution. 

\vspace{-0.5em}
\paragraph{Vorticity Data.} 
We employ a pseudo-spectral method to simulate~\eqref{eq:2D_vorticity_NS} and hence to obtain a dataset of snapshots of the vorticity field. 
We set the timestep $\Delta t = 10^{-4}$ and grid size to $256\times 256$. 
We store snapshots at regular intervals of $\Delta t = 0.5$. We conduct simulations for 2000 trajectories within the time range of $t \in [0,100]$; we then exclude the initial phase $t \in [0,50]$ from our data. 
Ultimately, we collect a total of $2\times 10^5$ snapshots, which are treated as samples from the invariant measure of~\eqref{eq:2D_vorticity_NS}. 
To reduce memory requirements, we downsize the dataset to a resolution of $128\times 128$.
In all our experiments, we used a UNet~\cite{ho2020denoising} as our network for approximating the velocity field. 
Detailed parameters for the training and dataset generation can be found in the  Appendix~\ref{sec:detail:nse}.

\vspace{-0.5em}
\paragraph{Forecasting at Full Resolution.}  
First, we consider predicting the distribution of the vorticity field that may evolve from a given realization.
To this end, we learn the SDE that samples the conditional distribution of vorticity fields after lag $\tau = 0.5$, and we iterate this SDE to get forecasted predictions after lag $2\tau$, $3\tau$, etc. 
In the top row of Fig.~\ref{fig:cond-stats-NSE}, the first panel shows a snapshot of a vorticity field, while the next three panels show different samples of vorticity fields generated after lag $2$ by iterating our forecasting procedure. 
\begin{figure*}[t]
    \centering
    \begin{minipage}{0.14\linewidth}
    \begin{subfigure}
    \centering
    \begin{overpic}[width=1\linewidth]{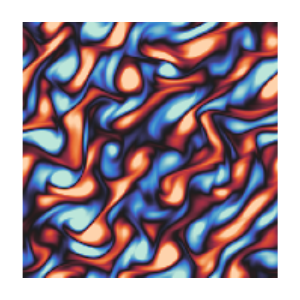}\put(45,-6){\small $\omega_{t}$}
    \end{overpic}
    \end{subfigure}\\
    \vspace{0.2em}
        \begin{subfigure}
        \centering
        \begin{overpic}[width=1\linewidth]{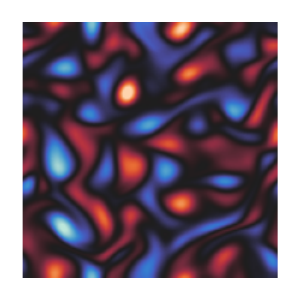}\put(13,-8){\small true v.s. forecast conditional mean}
        \end{overpic}
    \end{subfigure}
    \end{minipage}
    \begin{minipage}{0.14\linewidth}
    \begin{subfigure}
    \centering
    \begin{overpic}[width=1\linewidth]{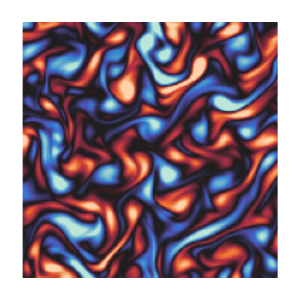}\put(80,-10){\small $\leftarrow$ three forecasts of $\omega_{t+2}\rightarrow$}\end{overpic}
    \end{subfigure}\\
    \vspace{0.2em}
        \begin{subfigure}
        \centering
        \includegraphics[width=1\linewidth]{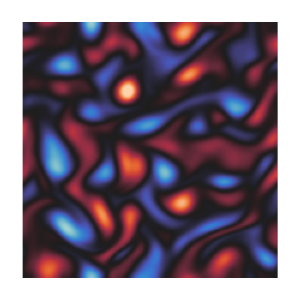}\\ 
    \end{subfigure}
    \end{minipage}
        \begin{minipage}{0.14\linewidth}
    \begin{subfigure}
    \centering
        \includegraphics[width=1\linewidth]{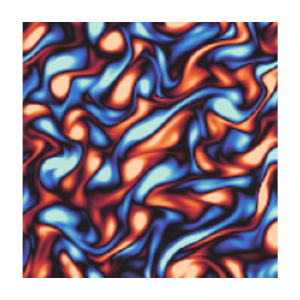}
    \end{subfigure}\\
    \vspace{0.2em}
        \begin{subfigure}
        \centering
        \begin{overpic}[width=1\linewidth]{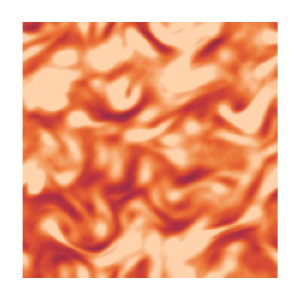}\put(14,-8){\small true v.s. forecast conditional std}
        \end{overpic}
    \end{subfigure}
    \end{minipage}
            \begin{minipage}{0.14\linewidth}
    \begin{subfigure}
    \centering
        \includegraphics[width=1\linewidth]{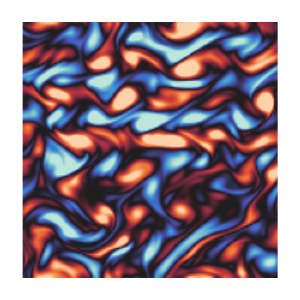}
    \end{subfigure}\\
    \vspace{0.2em}
        \begin{subfigure}
        \centering
        \includegraphics[width=1\linewidth]{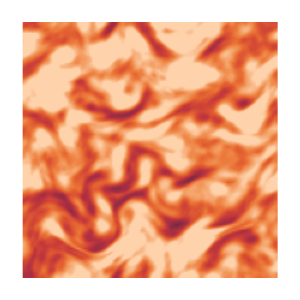}
    \end{subfigure}
    \end{minipage}
    \begin{minipage}{0.32\linewidth}
    \begin{subfigure}
    \centering
        \begin{overpic}[width=1\linewidth]{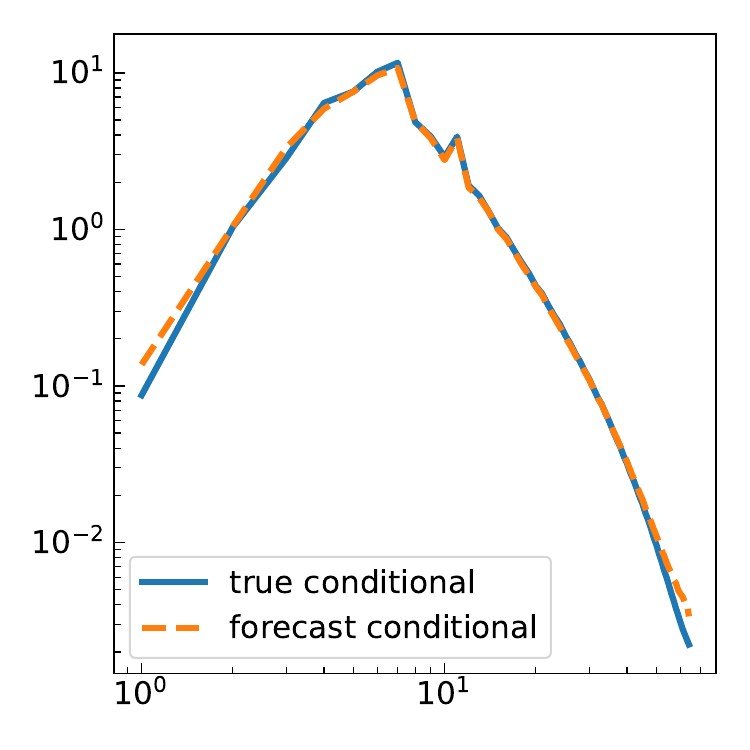}\put(26,-2){\small enstrophy spectrum (v.s. $k$)}\end{overpic}
    \end{subfigure}
    \end{minipage}
    \caption{\textbf{Temporal forecasting on stochastically-forced Navier Stokes.}
    (Top left) Different forecasts from our method at lag $\tau =2$, for a fixed $\omega_t$. 
    (Bottom left) Comparisons between the forecast sample mean and standard deviation for this $\omega_t$ against the truths. 
    (Right) Enstrophy spectrum of the true and forecasted conditional distribution. 
    Note: All the NS figures in this paper share the same color bars for the vorticity field (on a scale from $-5$ to $5$) and std (on a scale from $0$ to $3$) respectively.}
    \label{fig:cond-stats-NSE}
    \end{figure*}
Note that these generated vorticity fields are different from one another, emphasizing the need for probabilistic forecasting. 
This is corroborated by the true and forecasted conditional means of the field shown in the first two panels on the bottom row of Fig.~\ref{fig:cond-stats-NSE}.
While correctly captured by our approach, this conditional mean is clearly not informative on its own, as it averages over the spatial features present in the actual forecast. 
The spread of this ensemble of forecasts is also apparent from the standard deviation of the field, shown in the third and fourth panel on the bottom row of Fig.~\ref{fig:cond-stats-NSE}. 
Also shown in the right panel is the enstrophy spectrum of the true vorticity field and the ensemble of forecasts (see~\Cref{sec:detail:nse}), showing that our method captures this important physical quantity correctly. 
 \begin{figure}
    \begin{minipage}{.12\textwidth}
  \begin{subfigure}
    \centering
    \begin{overpic}[width=\linewidth]{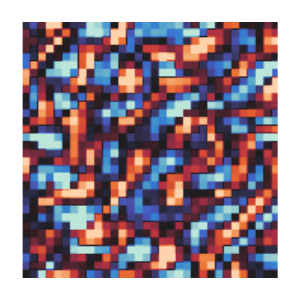}\put(22,-4){\scriptsize $32\times 32$ $\omega_t$}
    \end{overpic}
  \end{subfigure}\\
  \begin{subfigure}
    \centering
    \begin{overpic}[width=\linewidth]{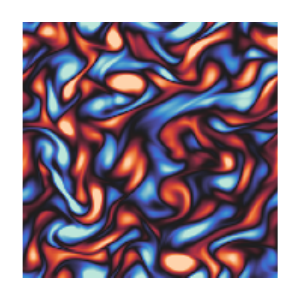}\put(16,-4){\scriptsize forecast $\omega_{t+1}$}
    \end{overpic}
  \end{subfigure}
\end{minipage}
\begin{minipage}{.12\textwidth}
  \begin{subfigure}
    \centering
    \begin{overpic}[width=\linewidth]{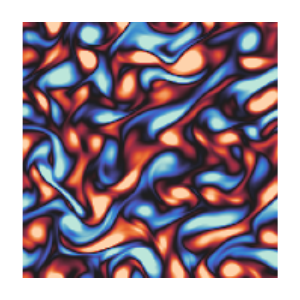}\put(16,-4){\scriptsize forecast $\omega_{t+1}$}\end{overpic}
  \end{subfigure}\\
  \begin{subfigure}
    \centering
    \begin{overpic}[width=\linewidth]{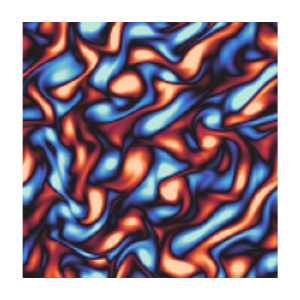}\put(16,-4){\scriptsize forecast $\omega_{t+1}$}
    \end{overpic}
  \end{subfigure}
\end{minipage}%
\begin{minipage}{.25\textwidth}
  \begin{subfigure}
    \centering
    \begin{overpic}[width=\linewidth]{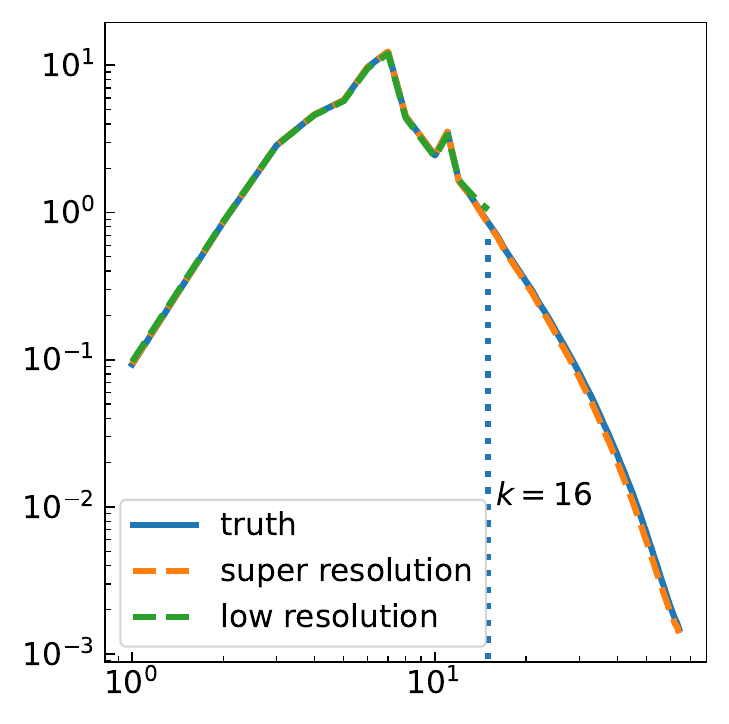}
        \put(25,-4){\scriptsize enstrophy spectrum (v.s. $k$)}
    \end{overpic}
  \end{subfigure}
  \end{minipage}
  \vspace{-1.0em}
    \caption{
    \textbf{Spatiotemporal forecasting on stochastically-forced Navier Stokes.}
    (Left) Low resolution $\omega_t$ and three of our forecasted samples at $t+1$. %
    (Right) Enstrophy spectrums of the low resolution $\omega_t$, the superresolution forecast of $\omega_{t+1}$, and the true $\omega_{t+1}$.}
    \label{fig:lowres:highres:nse}
    \end{figure}

\vspace{-0.5em}
\paragraph{Forecasting from Low-Resolution Data.} 
We now consider forecasting the vorticity field at a resolution of $128\times 128$ from a downsized version at resolution $32\times 32$ by learning the SDE as before with a drift velocity that is conditioned on the low-resolution field used as $x_0$.
The outcome of this task is shown in Fig.~\ref{fig:lowres:highres:nse}, where the four panels on the right show the low-resolution field used as input surrounded by three full-resolution forecasts generated after lag $=1$. 
We also plot the enstrophy spectrum of the low-resolution field
(which does not go past wavenumber 16) and the spectra of the true and the forecasted field at full resolution in the left panel. 
As can be seen, our approach recovers the true spectrum very accurately. 

We present additional experiments such as superresolution, comparisons between using $\sigma_s$ and $g^{\rm F}_s$ (F\"ollmer process) in terms of KL accuracy, and comparisons to flow matching and deterministic forecasting in Appendix~\ref{sec:detail:nse}. In addition, in Appendix \ref{sec:detail:nse}, we present results showing that the SDE forecasting can be $100\times$ faster than directly simulating the stochastic PDE in the Navier-Stokes example. 
\subsection{Video Forecasting}
We model the KTH and CLEVRER datasets. 
We follow RIVER \cite{davtyan2023efficient} and model these videos in the latent space of a VQGAN 
\cite{esser2021taming}
trained to auto-encode the datasets,
as is common for high-resolution image and video synthesis \cite{vahdat2021score, rombach2022high, peebles2023scalable,blattmann2023stable, davtyan2023efficient, ma2024sit}.

\vspace{-0.5em}
\paragraph{Task Description.} 
For video frames $x^{t}$
with $C$ channels and resolution $H \times W$, 
the VQGAN maps each video frame
to a latent image $y^{t} = \text{Encode}(x^t)$ with 
$C_\ell$ latent channels and latent resolution $H_\ell \times W_\ell$.
Our aim is to sample from the conditional density $\rho_c(y^t | y^{t-1}, \hdots, y^{t-C})$ describing the probability of a frame given a sequence of previous frames, which can then be decoded $x^t = \text{Decode}(y^t)$.
We set the interpolant base distribution to a point mass on $y_{s=0} = y^{t-1}$ and we model $y_{s=1} = y^t$. 

\vspace{-0.5em}
\paragraph{Generation.} 
%
%
Because conditioning on the whole set of $C$ previous time slices is costly, we follow RIVER~\cite{davtyan2023efficient} and use a Monte-Carlo estimator that generates our estimate of the $t^{th}$ latent frame $\hat{y}^{t}$ conditional on latent frame $y^{t-1}$ and an additional frame $y^{t-j}$ randomly chosen for $1 \geq j < t-1$.
To give the network context for the conditioned frame, we also condition on the time index $t-j$.
This random conditioning set $(y^{t-1}, y^{t-j}, t-j)$ avoids the need to compute functions of the entire conditioning context.
Samples are then decoded to produce images $x^t = \text{Decode}(y^t)$.
To sample a full video, we apply this forecasting strategy autoregressively. 
For further details, see~\Cref{alg:videoEM} in Appendix~\ref{app:video}.

\begin{figure*}[ht]
    \centering    
    \includegraphics[width=.99\linewidth]{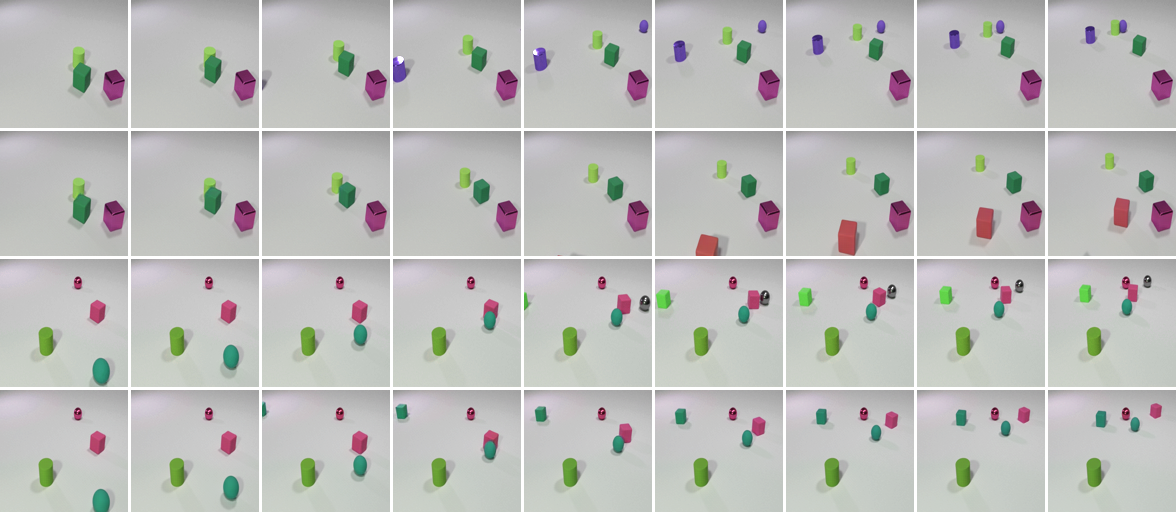}
    \caption{
    \textbf{Video generation on the CLEVRER dataset.}
    (Top row) Real trajectory.
    (Second row) Generated trajectory. A new, red cube enters the scene.
    (Third row) Real trajectory.
    (Fourth row) Generated trajectory. A new green cube enters the scene, and collision physics is respected (green ball hits red cube).}
    \label{fig:clevrer}
\end{figure*}
\begin{figure*}[ht]
    \centering    
    \includegraphics[width=.99\linewidth]{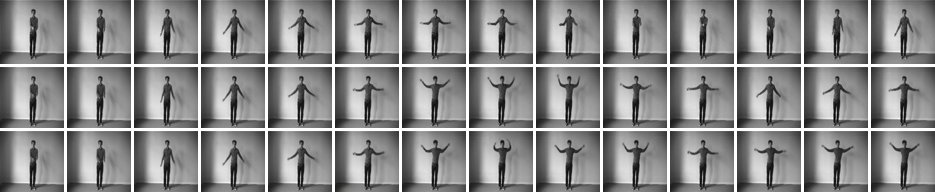}
    \caption{
    \textbf{Video generation on the KTH dataset.}
    (Top row) Real trajectory.
    (Middle and bottom rows) Generated trajectories. As time proceeds, the generated videos drift from the true video, but still display temporally-consistent hand-waving motions.
    }
    \label{fig:kth}
\end{figure*}

\vspace{-0.5em}
\paragraph{FVD metric.} 
The Fr\'echet Video Distance (FVD)~\cite{unterthiner2018towards} extends the Fr\'echet Inception Distance (FID)~\cite{heusel2017gans}. We select 256 test set videos and generate 100 completions for each one, thereby comparing 256 real videos to 25,600 generated videos. We also report several qualitative features of the generated videos. One consequence of using the VQGAN is that the performance is bounded by the FVD of the decoded encoded data since the generative model targets the encoded distribution. More information on evaluation is provided at
\href{https://github.com/interpolants/forecasting}{https://github.com/interpolants/forecasting}.

\vspace{-0.5em}
\paragraph{Baseline.}
We compare with the setup from \citet{davtyan2023efficient}, which learns a deterministic flow using flow matching~\cite{lipman2022flow, liu2022flow,albergo2022building} to map a Gaussian sample to the next video frame, conditioned on the same information as in our setup.
By contrast, we generate with an SDE sampler initialized at the previous video frame, which is more proximal to the next target frame than pure noise. 
We use the VQGAN checkpoints from RIVER so that we can study our proposed generative modeling method in a controlled context.


\vspace{-0.5em}
\paragraph{Datasets.}
The KTH dataset
\cite{schuldt2004recognizing}
consists of black-and-white videos of 25 people completing one of 6 actions such as jogging and hand-waving.
We use the last 5 people as the test set. 
The $1\times 64 \times 64$-dimensional data space is mapped to a $4 \times 8 \times 8$-dimensional latent space by the VQGAN.
During generation,  we start with 10 given video frames, and we generate the next 30 frames.
The CLEVRER dataset\footnote{\href{http://clevrer.csail.mit.edu/}{http://clevrer.csail.mit.edu/}} 
\cite{yi2019clevrer} contains videos created for studying reasoning and physics tasks. The videos feature cubes, spheres, and other shapes traveling across the screen and interacting through collisions while being subject to forces and rotations. 
One interpretation of this data is that a generative model needs to deduce physical phenomena to succeed at generation.
For example, objects should not go through each other, and instead should bounce off one another. 
The $3 \times 128 \times 128$-dimensional data space is mapped to a
$4 \times 16 \times 16$ latent space.
During generation, we condition on just 2 real frames and generate 14. 

\begin{table}[ht!]
    \centering
    \begin{tabular}{lllll}
    \toprule
     &\multicolumn{2}{c}{\textit{KTH}}
     &\multicolumn{2}{c}{\textit{CLEVRER}}
     \\
    \cmidrule{2-3}
    \cmidrule{4-5}
   \textit{Method}  & 100k & 250k &100k &250k\\
    \midrule
     RIVER &46.69& 41.88 &60.40 & 48.96 \\
     PFI (ours) & \textbf{44.38} & \textbf{39.13} & \textbf{54.7} & \textbf{39.31}\\
     \midrule 
      Auto-enc. & 33.45 & 33.45& 2.79 & 2.79\\
     \midrule
     \textit{Shifted FVD} &  &  &   &  \\
     \midrule
     RIVER & 13.24 & 8.43 & 57.61 & 46.17 \\
     PFI (ours) & \textbf{10.93} & \textbf{5.68} & \textbf{51.91} & \textbf{36.52}\\

     \bottomrule
    \end{tabular}
    \caption{\textbf{Video Results.} FVD computed on 256 test set videos, with the model generating 100 completions for each video. Results are reported for 100k and for 250k gradient steps. 
    ``Auto-enc.'' represents the FVD of the pretrained encoder-decoder compared to the real data. 
    It serves as a bound on the possible model performance, because the modeling is done in the latent space computed by the encoder-decoder pair.
    We also supply a shifted FVD, where encoded-decoded FVD is subtracted from FVD values to show they approach this approximate bound.}
    \label{tab:kth_clevrer}
\end{table}

\paragraph{Results.} Training details are in Appendix \ref{app:video}.
\cref{tab:kth_clevrer} shows the FVD performance of our model, probabilistic forecasting with interpolants (PFI), as compared to the RIVER baseline. We train both models under the same conditions for a controlled comparison. On both the KTH and CLEVRER datasets, PFI surpasses the standard flow matching approach. In addition to the numerical comparison, we demonstrate that our models produce diverse forecasts that capture physical rules inherent to the videos. In \cref{fig:clevrer}, we illustrate two trajectories of the animation. For each, we supply the initial condition, the dataset trajectory, and a generated trajectory based off of the same initial condition to show that the continuation of frames is probabilistic. For example, there is variation in the dataset trajectory of a green cube colliding with a green cylinder as compared to the forecasted trajectory, while also preserving the animated physics. In \cref{fig:kth}, we show that, from the same initial frame, the forecasts give varied realizations of the hand-waving video category.

\section{Conclusion and Future Work}
In this work, we introduced a principled approach to the use of generative modeling for probabilistic forecasting. 
By introducing stochastic processes that transport a point mass centered at a current observation of the system to a distribution over future states of that system, our proposed method uses dynamical measure transport in a way that  naturally aligns with the framework of probabilistic forecasting. 
It also allows us to minimize the impact of the estimation error by tuning of the diffusion coefficient, which can be done after training and offers a new perspective on the F\"ollmer process.
We have shown various uses of this approach, ranging from predicting the evolution of stochastic fluid dynamics to video completion tasks. Future work will consider using these models for empirical weather data and incorporation of physical structure into the generative model. Other generalizations will include mitigating the cost of simulating the SDE e.g. by tuning the diffusion coefficient to that effect, or using data sampled at random time-lags $\tau$  by conditioning the drift on these~$\tau$. 

\section*{Impact Statement}
This paper presents work whose goal is to advance the field of Machine Learning. 
There are many potential societal consequences of our work, none of which we feel must be specifically highlighted here. 
We note, however, that video generation, which is one of the possible applications of the forecasting framework, is a less explored domain that could promote harm through biases in the model. 
We surmise that it carries the same risks as image generation technologies.
\section*{Acknowledgments}
We thank Joan Bruna, Jon Niles-Weed, Loucas Pillaud-Vivien, and Valentin De Bortoli for useful discussions. YC, MSA and NMB are funded by the ONR project under the Vannevar Bush award ``Mathematical Foundations and Scientific Applications of Machine Learning''. MH is supported on the Meta Platforms, Inc project entitled Probabilistic Deep Learning with Dynamical Systems. EVE is supported by the National Science Foundation under Awards DMR-1420073, DMS-2012510, and DMS-2134216, by the Simons Collaboration on Wave Turbulence, Grant No. 617006, and by a Vannevar Bush Faculty Fellowship.

\bibliography{icml2024}
\bibliographystyle{icml2024}
\newpage
\appendix
\onecolumn
\section{Discussion on the stationarity assumption}
\label{appendix-discussion-stationarity}
The stationarity assumption of the time series in Section \ref{sec:setup} implies that $\rho(x_{k\tau}, x_{(k+1)\tau})$ is independent of time. This assumption is sufficient, but not necessary for the use of our algorithm. More specifically, we can decompose the possible scenarios into 4 categories:
\begin{table}[h]
\centering
\begin{tabular}{ccc}
\hline
  & Markovian & Non-Markovian \\ \hline
Stationary& Works as currently presented & Works by the stationarity assumption  \\
Non-stationary & Works; learn a time-independent transition kernel & Works; need to condition on time in the model \\ \hline
\end{tabular}
\end{table}
\\
In more details:
\begin{itemize}
    \item If the time series is stationary and Markovian, then the method works as it is currently presented.
\item If the time series is non-stationary but Markovian, our algorithm can still accurately learn the conditional distribution $\rho_c(\cdot|x_0)$, provided that the transition kernel is time-independent. This can be seen from our formulation, which only requires a joint distribution $\rho(x_0,x_1)$ to be time-independent. We have the flexibility to use any distribution for the input $x_0$ and use it with the conditional distribution $\rho_c(\cdot|x_0)$ to define the joint distribution.
\item If the time series is stationary but non-Markovian, then the method works again by the stationarity assumption.
\item If the time series is non-stationary and non-Markovian, then the method can still work if it is adapted so that the drift also conditions on the physical time.
\end{itemize}
The Navier-Stokes experiments we present are in the category `stationary and Markovian' since to obtain the training data, we ran the simulation for a sufficiently long time and discarded the initial samples. 
The video experiments  seem to belong to the category  `non-stationary but Markovian', though this assertion is hard to test.

\section{Details on stochastic interpolants}

\subsection{Analytical formulas for some specific $\alpha_s,\beta_s,\sigma_s$}
\label{app:specifics}

In this section, we present the formulas for some specific choices of $\alpha_s,\beta_s$, and $\sigma_s$, the corresponding optimal (F\"ollmer) drift $g^{\rm F}_s$, and the expression for $b^g_s$ in terms of $b_s$.

For $\alpha_s= 1- s, \sigma_s = \varepsilon(1-s)$ where $\varepsilon>0$ is a tunable parameter, and $\beta_s=s$, we have
\begin{equation}
    \label{eq:b:c1}
    \begin{aligned}
        b_s(x,x_0) &= \E^{x_0} [ x_1-x_0 -\varepsilon\sqrt{s} z| (1-s) x_0 + s x_1 + \varepsilon (1-s) \sqrt{s} z = x], \qquad \text{for} \ \ s\in(0,1)\\
        b_0(x_0,x_0)  &= \E^{x_0} [x_1] - x_0,\\
        b_1(x,x_0)  &= x - x_0.
    \end{aligned}
\end{equation}
Moreover, $b_s$ is the minimizer of the loss (with respect to $\hat{b}$)
\begin{equation}
    \label{eq:b:loss:c1}
    L_b[\hat b] = \int_0^1 \E\big[|\hat b_s((1-s) x_0 + s x_1 + (1-s) \varepsilon\sqrt{s} z,x_0) -( x_1-x_0 -\varepsilon\sqrt{s} z) |^2 \big] ds.
\end{equation}
We also have
\begin{equation}
    \label{eq:A:c:c1}
    A_s = \frac{1}{\varepsilon^2 s(1-s)}, \qquad c_s(x,x_0) = x - x_0
\end{equation}
so that, for any $g_s$ such that $\lim_{s\to0^+} s^{-1}[g^2_s-\varepsilon^2(1-s)^2]$ and $\lim_{s\to1^-} g^2_s/(\varepsilon(1-s))$ exists, 
\begin{equation}
    \label{eq:bg:c1}
    \begin{aligned}
    b^g_s(x,x_0)  &= b_s(x,x_0) +\tfrac12 \big(g_s^2-\varepsilon^2(1-s)^2\big) \left( \frac{b_s(x,x_0)}{\varepsilon^2(1-s)} - \frac{x-x_0}{\varepsilon^2 s(1-s)}\right) \qquad \text{for} \ \ s\in(0,1),\\
    b^g_0(x_0,x_0)  &= b_0(x_0,x_0) = \E^{x_0} [x_1] - x_0,\\
    b^g_1(x,x_0)  &= b_1(x,x_0) = x - x_0.
    \end{aligned}
\end{equation}
In addition, we have
\begin{equation}
    \label{eq:gf:c1}
    g^\Fo_s = \varepsilon\sqrt{(1-s)(1+s)},
\end{equation}
so that
\begin{equation}
    \label{eq:bF:c1}
    b^{\Fo}_s(x,x_0)  = (1+s) b_s(x,x_0) - x+x_0\qquad \text{for} \ \ s\in[0,1].
\end{equation}

For $\alpha_s= 1- s, \sigma_s = \varepsilon(1-s)$, and $\beta_s=s^2$, we have
\begin{equation}
    \label{eq:b:c2}
    \begin{aligned}
        b_s(x,x_0) &= \E^{x_0} [ 2sx_1-x_0 -\varepsilon\sqrt{s} z| (1-s) x_0 + s^2 x_1 + \varepsilon(1-s) \sqrt{s} z = x], \quad \text{for} \ \ s\in(0,1)\\
        b_0(x_0,x_0)  &= - x_0,\\
        b_1(x,x_0)  &= x - x_0.
    \end{aligned}
\end{equation}
Moreover, $b_s$ is the minimizer of the loss 
\begin{equation}
    \label{eq:b:loss:c2}
    L_b[\hat b] = \int_0^1 \E\big[|\hat b_s((1-s) x_0 + s^2 x_1 + 
\varepsilon(1-s) \sqrt{s} z,x_0) -( 2sx_1-x_0 -\varepsilon\sqrt{s} z) |^2 \big] ds.
\end{equation}
We also have
\begin{equation}
    \label{eq:A:c:c2}
    A_s = \frac{1}{\varepsilon^2 s^2(1-s)(2-s)}, \qquad c_s(x,x_0) = 2sx - s(2-s)x_0,
\end{equation}
so that, for any $g_s$ such that $\lim_{s\to0^+} s^{-1}[g^2_s-\varepsilon^2(1-s)^2]$ and $\lim_{s\to1^-} g^2_s/(\varepsilon(1-s))$ exist, 
\begin{equation}
    \label{eq:bg:c2}
    \begin{aligned}
    b^g_s(x,x_0)  &= b_s(x,x_0) +\tfrac12 \big(g_s^2-\varepsilon^2(1-s)^2\big) \left( \frac{b_s(x,x_0)}{\varepsilon^2(1-s)(2-s)} - \frac{x-x_0}{\varepsilon^2s^2(1-s)(2-s)}\right) \quad \text{for} \ \ s\in(0,1)\\
    b^g_0(x_0,x_0)  &= b_0(x_0,x_0) = \E^{x_0} [x_1] - x_0\\
    b^g_1(x,x_0)  &= b_1(x,x_0) = x - x_0.
    \end{aligned}
\end{equation}
In addition we have
\begin{equation}
    \label{eq:gf:c2}
    g^\Fo_s = \varepsilon\sqrt{(1-s)(3-s)}
\end{equation}
so that
\begin{equation}
    \label{eq:bF:c2}
    \begin{aligned}
    b^{\Fo}_s(x,x_0)  &= \Big(1+\frac{1}{2-s}\Big)  b_s(x,x_0) - \frac{1}{s(2-s)}(2x-(2-s)x_0)\qquad \text{for} \ \ s\in(0,1],\\
    b^{\Fo}_0(x_0,x_0)  &= -2x_0.
    \end{aligned}
\end{equation}

We summarize the above calculations in the following table:
\begin{table}[h]
\centering
\begin{tabular}{ccccccc}
\hline
 $\alpha_s$ & $\beta_s$ & $\sigma_s$  & $g^{\rm F}_s$  & $A_s$ & $c_s(x,x_0)$ & $b_s^\Fo(x,x_0)$\\ \hline
 $1-s$& $s$ & $\varepsilon(1-s)$ & $\varepsilon\sqrt{(1-s)(1+s)}$ &  $\frac{1}{\varepsilon^2s(1-s)}$ & $x-x_0$ & $(1+s) b_s - x + x_0 $\\
 $1-s$ & $s^2$ & $\varepsilon(1-s)$ & $\varepsilon\sqrt{(3-s)(1-s)}$ & $\frac{1}{\varepsilon^2s^2(1-s)(2-s)}$ &$2sx-s(2-s)x_0$ & $(1+\frac{1}{2-s}) b_s - \frac{2x - (2-s)x_0}{s(2-s)}$ \\ \hline
\end{tabular}
\end{table}

\subsection{Proof of Theorem~\ref{th:2:gen}}
\label{app:proof:th1}

Recall that:
\begin{definition}
\label{def:def1}
The stochastic interpolant $I_s$ is the stochastic process defined as 
\begin{equation}
    \label{eq:stochinterpolant-appendix}
    I_s = \alpha_s  x_0 + \beta_s  x_1 + \sigma_s W_s\qquad s\in[0,1],
\end{equation}
where
\begin{itemize}[leftmargin=0.15in]
\item  $\alpha,\beta,\sigma\in C^1([0,1])$ satisfy $\alpha_s^2+\beta^2_s+\sigma^2_s>0$ for all $s\in[0,1]$,  $\dot \beta_s>0$ for all $s\in(0,1]$, and $\dot\sigma_s <0$ for all $s\in[0,1]$, as well as the boundary conditions $\alpha_0=\beta_1=1$,  $\alpha_1=\beta_0=\sigma_1=0$.
\item The pair $(x_0,x_1)$ are jointly drawn from a PDF $\rho(x_0,x_1)$ such that 
$\E_{(x_0,x_1)\sim \rho}\big[|x_0|^2+|x_1|^2\big] <\infty$.
\item $W=(W_s)_{s\in[0,1]}$ is a standard Wiener process with $W\perp(x_0,x_1)$.
\end{itemize}
\end{definition}
In view of this definition, let us give a more precise formulation of Theorem~\ref{th:2:gen}:
\begin{theorem}
\label{thm:1:b}
   Let $I_s$ be the stochastic interpolant introduced in Definition~\ref{def:def1} and let
\begin{equation}
\label{eq:b:def:app}
\begin{aligned}
    &\forall s \in [0,1]: \quad &&x_s  = \alpha_s  x_0 + \beta_s  x_1 +  \sigma_s \sqrt{s} z \\
    &\forall (s,x,x_0) \in (0,1] \times \R^d \times\R^d : &&  b_s(x,x_0) = \mathbb{E}^{x_0}[\dot\alpha_s x_0 + \dot\beta_s x_1+ \dot \sigma_s  \sqrt{s} \,z|x_s  = x]
\end{aligned}
\end{equation}
where $\E^{x_0}[\cdot|x_s  = x]$ denotes an expectation over $x_1 \sim \rho_c(\cdot |x_0)$ and $z\sim \mathsf {\sf N}(0,\Id)$ conditional on $x_s  = x$. Moreover, set $b_{s=0}(x_0,x_0) := \dot\alpha_0 x_0 + \dot\beta_0 \mathbb{E}^{x_0}[x_1]$.
Then the solutions to the SDE
\begin{equation}
    \label{eq:sde:app}
    dX_s = b_s(X_s,x_0) ds + \sigma_s  dW_s, \qquad X_{s=0} = x_0,
\end{equation}
are such that $\text{Law}(X_s) = \text{Law}(I_s|x_0) $ at all $(s,x_0)\in [0,1]\times \R^d$. In particular $X_{s=1}\sim \rho_c(\cdot |x_0)$. In addition the drift $b$ is the unique minizer over all $\hat b$ of the objective function
\begin{equation}
\label{eq:loss:b:app}
L_b[\hat b] = \int_0^1 \E \big[|\hat b_s(x_s,x_0) - ( \dot \alpha_sx_0 + \dot \beta_s  x_1+ \dot\sigma_s  \sqrt{s} \,z)|^2]  ds,
\end{equation}
where $\E$ denotes an expectation over $(x_0,x_1)\sim \rho(x_0,x_1)$ and $z\sim {\sf N}(0,\Id)$ with $(x_0,x_1)\perp z$.
\end{theorem}

Note that the objective~\eqref{eq:loss:b:app} is the same as~\eqref{eq:loss} because $x_s \stackrel{d}{=} I_s$ and $\dot \alpha_sx_0 + \dot \beta_s  x_1+ \dot\sigma_s  \sqrt{s} \,z \stackrel{d}{=} R_s$ at all $s\in [0,1]$.
Note also that, since $\alpha,\beta,\sigma\in C^1([0,1])$, the factors $\dot\alpha_s , \dot\beta_s , \dot\sigma_s \sqrt{s}$ in the loss~\eqref{eq:loss:b:app} are all bounded.

\begin{proof}[Proof of Theorem~\ref{thm:1:b}.] 
Notice that the process $I_s$ defined in~\eqref{eq:stochinterpolant-appendix} has the same law at any $s\in[0,1]$ as  $x_s $.
If we denote by $\mu(s,dx|x_0)$ the measure of $I_s|x_0$ or $x_s |x_0$, and by $\phi:\R^d \to \R$ a twice-differentiable test function with compact support, by definition we have
\begin{equation}
    \label{eq:chracteristic}
    \forall (s,x_0)\in [0,1]\times\R^d \quad : \quad \int_{\R^d} \phi(x) \mu(s,dx|x_0)  = \E[\phi(x_s )|x_0] = \E[\phi(I_s)|x_0],
\end{equation}
where the first conditional expectation is over $(x_0,x_1)\sim \rho$ and $z\sim \sf {\sf N}(0,\Id)$ and the second over $(x_0,x_1)\sim \rho$ and $W$.
By the It\^o formula we have
\begin{equation}
\label{eq:dcharact}
d \phi( I_s) = (\dot\alpha_s  x_0 + \dot \beta_s  x_1 + \dot \sigma_s  W_s) \cdot \nabla \phi( I_s) ds + \tfrac12\sigma_s^2 \Delta \phi(I_s) ds + \nabla \phi(I_s) \cdot dW_s.
\end{equation}
Integrating this equation in time over $[0,s]$, taking the expectation conditional on $x_0$, and using both $W_s\perp(x_0,x_1)$ and the It\^o isometry, we deduce that
\begin{equation}
\label{eq:dcharact:E}
\E[\phi( I_s)|x_0]  =  \phi(x_0) + \int_0^s \left(\E\big[(\dot\alpha_r x_0 + \dot \beta_r x_1 + \dot \sigma_r W_r) \cdot \nabla \phi(I_r) \big| x_0 \big]  + \tfrac12\sigma^2_r \E[\Delta \phi(I_r)|x_0] \right) dr.
\end{equation}
Inserting~\eqref{eq:dcharact:E} into~\eqref{eq:chracteristic},  we deduce that, $\forall (s,x_0)\in [0,1]\times\R^d$,
\begin{equation}
\label{eq:t:derv}
\begin{aligned}
    \int_{\R^d} \phi(x) \mu(s,dx|x_0)  &=  \phi(x_0)+ \int_0^s \left(\E\big[(\dot\alpha_r x_0 + \dot \beta_r x_1 + \dot \sigma_r \sqrt{r} \, z ) \cdot \nabla \phi(x_r) \big| x_0 \big] \right.\\
    & \qquad \qquad\qquad  \left.  + \tfrac12\sigma^2_r \E[\Delta \phi(x_r)|x_0]\right) dr,
\end{aligned}
\end{equation}
where we used the fact that $x_s$ and $I_s$ share the same law at each $s$. Also, $W_s$ and $\sqrt{s}z$ share the same law at each $s$.

Using the tower property of the conditional expectation, \eqref{eq:t:derv} can also be written as 
\begin{equation}
\label{eq:t:derv:2}
\begin{aligned} 
    &\int_{\R^d} \phi( x) \mu(s,dx|x_0)\\
    &= \phi(x_0) + \int_0^s \int_{\R^d}  \E\big[(\dot\alpha_r x_0 + \dot \beta_r x_1 + \dot \sigma_r \sqrt{r}\, z) \cdot \nabla \phi(x_r)\big| x_r = x,x_0\big] \mu(r,dx|x_0) dr \\
    & \qquad \qquad + \tfrac12 \int_0^s \sigma^2_r \int_{\R^d}\Delta\phi(x)  \mu(r,dx|x_0) dr \\
    & = \phi(x_0) + \int_0^s \int_{\R^d}  \E\big[(\dot\alpha_r x_0 + \dot \beta_r x_1 + \dot \sigma_r \sqrt{r}\, z) \big| x_r = x,x_0\big] \cdot \nabla \phi(x)\mu(r,dx|x_0) dr \\
    & \qquad \qquad + \tfrac12 \int_0^s \sigma^2_r \int_{\R^d} \Delta \phi(x)  \mu(r,dx|x_0) dr \\
    &= \phi(x_0) + \int_0^s \int_{\R^d}  \left(  b_r(x,x_0) \cdot \nabla \phi(x)  + \tfrac12 \sigma^2_r  \Delta\phi(x) \right) \mu(r,dx|x_0) 
    \end{aligned}
\end{equation}
where we used the definition of $b$ in~\eqref{eq:b:def:app} to get the last equality. 
If we now repeat the same steps  to derive an evolution equation for $\E[\phi(X_s)]$ where $X_s$ solves the SDE~\eqref{eq:sde} 
we arrive at the same equation~\eqref{eq:t:derv:2} for the measure of this process, indicating that this measure is also $ \mu(s,dx|x_0)$.

It remains to show that the drift~$b_s(x,x_0) = \mathbb{E}[\dot\alpha_s x_0 + \dot\beta_s x_1+ \dot \sigma_s  \sqrt{s} \,z|x_s  = x, x_0]$ is the unique minimizer of the objective function~\eqref{eq:loss}.
To this end, notice that $b$ is the unique minimizer over all $\hat b$ of
\begin{equation}
    \begin{aligned}
        &\int_0^1 \E \big[|\hat b_s(x_s,x_0) - b_s(x_s,x_0)|^2ds,\\
         = &\int_0^1 \left(\E \big[|\hat b_s(x_s,x_0) - ( \dot \alpha_s x_0 + \dot \beta_s  x_1+ \dot\sigma_s  \sqrt{s} \,z)|^2 + \mathrm{var}^{x_0}[\dot\alpha_s x_0 + \dot\beta_s x_1+ \dot \sigma_s  \sqrt{s} \,z|x_s] \big]\right)ds,
    \end{aligned}
\end{equation}
where $\E$ denotes an expectation over $(x_0,x_1)\sim \rho(x_0,x_1)$ and $z\sim {\sf N}(0,\Id)$ with $(x_0,x_1)\perp z$ and where
\begin{equation}
\label{eq:var:def}
\begin{aligned}
\mathrm{var}^{x_0}[\dot\alpha_s x_0 + \dot\beta_s x_1+ \dot \sigma_s  \sqrt{s} \,z|x_s]
&= \E^{x_0}\big[|\dot\alpha_s x_0 + \dot\beta_s x_1+ \dot \sigma_s  \sqrt{s} \,z|^2|x_s\big] - 
\big|\E^{x_0}[\dot\alpha_s x_0 + \dot\beta_s x_1+ \dot \sigma_s  \sqrt{s} \,z|x_s]\big|^2,\\
& = \E^{x_0}\big[|\dot\alpha_s x_0 + \dot\beta_s x_1+ \dot \sigma_s  \sqrt{s} \,z|^2|x_s\big] -|b_s(x_s,x_0)|^2,
\end{aligned}
\end{equation}
and where we used the tower property of the conditional expectation $\E[\E[|\dot\alpha_s x_0 + \dot\beta_s x_1+ \dot \sigma_s  \sqrt{s} \,z|^2|x_s]] = \E[|\dot\alpha_s x_0 + \dot\beta_s x_1+ \dot \sigma_s  \sqrt{s} \,z|^2|x_s] $.
Since $\mathrm{var}^{x_0}[\dot\alpha_s x_0 + \dot\beta_s x_1+ \dot \sigma_s  \sqrt{s} \,z|x_s  = x]$ is independent of $\hat{b}$, we can drop this term to arrive at the objective~\eqref{eq:loss:b:app}. 
\end{proof}

\subsection{Regularity of the drift at $s=0$}
\label{app:reg}
In this section, we discuss the regularity of the drift $b_s(x,x_0)$.
\begin{assumption} 
    The conditional distribution $\rho_c(\cdot|x_0)$ is exponential tailed. More precisely, there exist constants $C_1,C_2 > 0$ (which can depend on $x_0$), such that
    \[\rho_c(x|x_0)\leq C_2\exp(-C_1|x|)\, , \]
    for any $x \in \mathbb{R}^d$. 
    \label{assump:exponential-tails}
\end{assumption}
The aforementioned assumption is needed for technical reasons, and is used to ensure the validity of a step that involves the interchange of limits and integrations in the proof of Theorem \ref{th:decokmp}.

By linearity of $x_s$ in $x_0$, $x_1$, and $z$, we can also establish a few properties of the velocity $b$ in \eqref{eq:b:def:app} which we state as:
\begin{theorem}
\label{th:decokmp}
    Under Assumption \ref{assump:exponential-tails}, if  $\dot\beta_{s=0} = 0$, then the velocity field $b_s(x,x_0) $ can be decomposed as
    \begin{equation}
    \label{eq:decompos}
    \forall (s,x,x_0) \in [0,1]\times \R^d \times \R^d \: \quad b_s(x,x_0) = \dot \alpha_s  x_0 + \dot\beta_s  \eta_1(s,x,x_0) + \dot \sigma_s   \eta_z(s,x,x_0),
    \end{equation}
    where
\begin{equation}
    \label{eq:eta1:z}
    \begin{aligned}
        \forall (s,x,x_0) \in (0,1]\times \R^d \times \R^d &:&  &\left\{ \begin{aligned} \eta_1(s,x,x_0) &= \E^{x_0}[x_1 | x_s = x], \\ 
        \eta_z(s,x,x_0)& = \sqrt{s}\,\E^{x_0}[z | x_s = x],
        \end{aligned} \right.\\
        \forall (x,x_0) \in \R^d \times \R^d &:&  &\left\{ \begin{aligned}
        \eta_1(0,x,x_0) &= \lim_{s\to0} \eta_1(s,x,x_0) = \E^{x_0}[x_1 ], \\
        \eta_z(0,x,x_0) &= \lim_{s\to0} \eta_z(s,x,x_0) = 
    \frac{x-x_0}{\sigma_0}.
    \end{aligned} \right.
    \end{aligned}
\end{equation}
In addition, these two functions satisfy the constraint
\begin{equation}
    \label{eq:constraint}
    \forall (s,x,x_0) \in [0,1]\times \R^d \times \R^d \quad : \quad x = \alpha_s  x_0 + \beta_s  \eta_1(s,x,x_0) + \sigma_s  \eta_z(s,x,x_0).
\end{equation}
We have $\lim_{s \to 0} b_s(x,x_0) = \dot\alpha_0x_0 + \frac{\dot\sigma_0}{\sigma_0}(x-x_0)$ and $\lim_{s \to 0} \nabla_x b_s(x,x_0) = \frac{\dot\sigma_0}{\sigma_0} \Id$ for any $x,x_0 \in \mathbb{R}^d$.
\end{theorem}
\begin{remark}
    Note that \eqref{eq:constraint} implies that we can get $\eta_1$ from $\eta_z$ at any time such that $\beta_s \not=0$, and $\eta_z$ from $\eta_1$ at any $s \in (0,1]$ such that $\sigma_s >0$: in particular
\begin{equation}
    \label{eq:etaz:1}
    \forall (s,x,x_0) \in (0,1]\times \R^d \times \R^d \ \text{with} \ \sigma_s >0 \ : \quad \eta_z(s,x,x_0) = \frac{x-\alpha_s  x_0- \beta_s   \eta_1(s,x,x_0)}{\sigma_s }.
\end{equation}
\end{remark}
The proof of Theorem~\ref{th:decokmp} will rely the following result:
\begin{lemma}
    \label{em:1}
    We have 
\begin{equation}
    \label{eq:x1:z}
    \begin{aligned}
        \forall (s,x,x_0)\in (0,1) \times \R^d \times \R^d: \quad \eta_1(s,x,x_0) = \frac{\int_{\R^d} x_1 \rho_c(x_1|x_0) e^{-\frac12M_s |x_1|^2 +m_s x_1 \cdot(x-\alpha_s x_0) } dx_1}{\int_{\R^d} \rho_c(x_1|x_0) e^{-\frac12M_s |x_1|^2 + m_s x_1 \cdot(x-\alpha_s x_0) }dx_1}
    \end{aligned}
\end{equation}
where $\eta_1(s,x,x_0) = \E^{x_0}[x_1|x_s=x,x_0]$ and we defined
\begin{equation}
    \label{eq:AB}
    M_s = \frac{\beta^2_s}{s \sigma^2_s}, \qquad m_s = \frac{\beta_s }{s \sigma^2_s}.
\end{equation}
\end{lemma}

\begin{proof}[Proof:]
By definition, 
\begin{equation}
    \label{eq:def:cond}
    \eta_1(s,x,x_0)  = \frac{\int_{\R^d} x_1 \rho_c(x_1|x_0) e^{-\frac12|z|^2 } \delta(x-\alpha_s  x_0 - \beta_s  x_1 - \sigma_s  \sqrt{s} z) dzdx_1}{\int_{\R^d} \rho_c(x_1|x_0) e^{-\frac12|z|^2 } \delta(x-\alpha_s  x_0 - \beta_s  x_1 - \sigma_s  \sqrt{s} z) dzdx_1},
\end{equation}
where $\delta(x)$ denotes the Dirac delta distribution.
For any $s \in (0,1]$ such that $\sigma_s >0$, we can perform the integration over $z$ explicitly to get
\begin{equation}
    \label{eq:x1:z:2}
    \begin{aligned}
        \eta_1(s,x,x_0) = \frac{\int_{\R^d} x_1 \rho_c(x_1|x_0) e^{-\frac12M_s |x_1|^2 +m_s x_1 \cdot(x-\alpha_s x_0) } e^{-\frac12 s^{-1}\sigma^{-2}_s|x-\alpha_s x_0|^2 } dx_1}{\int_{\R^d} \rho_c(x_1|x_0) e^{-\frac12M_s |x_1|^2 +m_s x_1 \cdot(x-\alpha_s x_0)} e^{ -\frac12 s^{-1}\sigma^{-2}_s|x-\alpha_s x_0|^2 }dx_1}.
    \end{aligned}
    \end{equation}
    Since the factors~$e^{-\frac12 s^{-1}\sigma^{-2}_s|x-\alpha_s x_0|^2}$ at the numerator and the denominator do not depend on $x_1$, they can be taken out of the integrals and simplified, and we arrive at~\eqref{eq:x1:z}.
\end{proof}

\begin{proof}[Proof of Theorem~\ref{th:decokmp}]
   The  only statement that is not a direct consequence of Theorem~\ref{thm:1:b} is \eqref{eq:eta1:z}. To establish these limits, notice that, since  $\beta\in C^2([0,1])$ and $\dot \beta_0 = 0$, we must have $\beta_s  = O(s^2)$ as $s\to 0$. As a result, since $\sigma\in C^1([0,1])$ and $\sigma_0>0$, we have 
\begin{equation}
    \label{eq:AB:lim}
    \lim_{s\to0} M_s = \lim_{s\to0} m_s = 0.
\end{equation}
As a result,
\begin{equation}
    \label{eq:x1:lim}
    \forall (x,x_0)\in \R^d \times \R^d \ : \quad \lim_{s\to0} \eta_1(s,x,x_0) = \E^{x_0}[x_1], \qquad \lim_{s\to0} \eta_z(s,x,x_0) = \frac{x-x_0}{\sigma_0},
\end{equation}
which establishes the limits in \eqref{eq:eta1:z}, where we used the formula in \eqref{eq:etaz:1} to derive the limits of $\eta_z$ from that of $\eta_1$. Note that in the above derivation, we need to verify the interchange of limits and integrations; it is guaranteed by using Assumption \ref{assump:exponential-tails} and the Lebesgue dominated convergence theorem  since, for a fixed $x, x_0$, when $s$ is sufficiently small, the factor $\rho_c(x_1|x_0) e^{-\frac12M_s |x_1|^2 +m_s x_1 \cdot(x-\alpha_s x_0)}$  is dominated by $\rho_c(x_1|x_0)e^{\frac12C_1|x|}$, which is integrable as a function of $x_1$ due to our Assumption \ref{assump:exponential-tails}.

To analyze the limit of $\nabla_x b_s(x,x_0)$, using Lemma \ref{em:1}, we have 
\begin{equation}
    \nabla_x\eta_1(s,x,x_0) = m_s\frac{P(s,x,x_0)Q(s,x,x_0) -R(s,x,x_0)R^T(s,x,x_0)}{Q(s,x,x_0)^2}, 
\end{equation}
where 
\begin{equation*}
    \begin{aligned}
        P(s,x,x_0) &= \int_{\R^d} x_1 x_1^T \rho_c(x_1|x_0) e^{-\frac12M_s |x_1|^2 +m_s x_1 \cdot(x-\alpha_s x_0) } dx_1,\\
        Q(s,x,x_0) & = \int_{\R^d} \rho_c(x_1|x_0) e^{-\frac12M_s |x_1|^2 +m_s x_1 \cdot(x-\alpha_s x_0) } dx_1,\\
        R(s,x,x_0) &= \int_{\R^d} x_1 \rho_c(x_1|x_0) e^{-\frac12M_s |x_1|^2 +m_s x_1 \cdot(x-\alpha_s x_0) } dx_1.
    \end{aligned}
\end{equation*}
To derive the above formula, we also need verify the interchange of limits and integrations; it is again guaranteed by using Assumption \ref{assump:exponential-tails} and the Lebesgue dominated convergence theorem, for sufficiently small $s$. We know that
\begin{equation}
    \lim_{s\to 0} \frac{P(s,x,x_0)Q(s,x,x_0) -R(s,x,x_0) R^T(s,x,x_0)}{|Q(s,x,x_0)|^2}  = \E[x_1 x_1^T|x_0] - \E[x_1|x_0] \E[x_1^T|x_0]\, ,
\end{equation}
and thus $\lim_{s \to 0} \nabla_x\eta_1(s,x,x_0) = 0$ as $\lim_{s\to 0} m_s=0$. Using the formula in \eqref{eq:etaz:1}, we get $\lim_{s \to 0} \nabla_x\eta_z(s,x,x_0) = \frac{1}{\sigma_0}\Id$. Therefore,
\begin{equation}
    \lim_{s\to 0} \nabla_x b_s(x,x_0) = \frac{\dot\sigma_0}{\sigma_0}\Id,
\end{equation}
which completes the proof.
\end{proof}

\subsection{Changing the Diffusion Coefficient: Proof of Theorem~\ref{th:1:gen}}
\label{app:tunable}

Theorem~\ref{th:1:gen} and formula~\eqref{eq:score} are consequences of the following result:
\begin{theorem}
\label{thm:tunning_g}
    Let~$g\in C^0([0,1])$ be such that  $\lim_{s\to0^+} s^{-1}[g^2_s -\sigma^2_s]$ and $\lim_{s\to1^-} g^2_s \sigma^{-1}_s$ exist. Let $b$ be given by~\eqref{eq:b:def:app} and define
\begin{equation}
    \label{eq:df:A:c-appendix}
    \begin{aligned}
    &\forall s\in (0,1) : \quad &&A_s = [s\sigma_s (\dot\beta_s \sigma_s -\beta_s  \dot\sigma_s )]^{-1}\\
    &\forall (s,x,x_0)\in (0,1)\times\R^d\times \R^d : &&c_s(x,x_0) = \dot \beta_s  x + (\beta_s  \dot \alpha_s  - \dot\beta_s  \alpha_s )x_0.
    \end{aligned}
\end{equation} 
Then we have 
\begin{equation}
    \label{eq:score:app}
    \forall (s,x,x_0)\in (0,1)\times\R^d\times \R^d : \quad \nabla \log \rho_s(x|x_0) = A_s \left[\beta_s b_s(x,x_0) - c_s(x,x_0) \right],
\end{equation}
where $\rho_s(x|x_0) $ denotes the PDF of $X_s\stackrel{d}{=} X^g_s \stackrel{d}{=} I_s|x_0$. In addition the drift
    \begin{equation}
\label{eq:drift:bc-appendix:a}
    \begin{aligned}
        b^g_s(x,x_0)  & = b_s(x,x_0) + \tfrac12 (g^2_s-\sigma^2_s) \nabla \log \rho_s(x|x_0) \\
        & = b_s(x,x_0) + \tfrac12 (g^2_s-\sigma^2_s) A_s \left[\beta_s b_s(x,x_0) - c_s(x,x_0) \right]
    \end{aligned}
\end{equation}
is well-defined for all $(s,x,x_0)\in (0,1)\times\R^d\times \R^d$, and has finite limits at $s=0,1$ for all $(x,x_0)\in \R^d\times \R^d$. Finally, the solutions to the SDE
\begin{equation}
    \label{eq:sde-appenfix}
    dX^g_s= b_s^g(X^g_s,x_0) ds + g_s  dW_s, \qquad X^g_{s=0} = x_0,
\end{equation}
are such that $\text{Law}(X^g_s)= \text{Law}(X_s) = \text{Law}(I_s|x_0) $ for all $(s,x_0)\in [0,1]\times \R^d$. In particular $X^g_{s=1}\sim \rho_c(\cdot |x_0)$.
\end{theorem}
\begin{proof}
Let us first establish~\eqref{eq:score:app}. 
By a direct extension of Stein's formula~\cite{albergo2023stochastic}, we have
\begin{equation}
\label{eq:stein}
\forall (s,x,x_0) \in (0,1) \times\R^d\times\R^d: \quad \nabla \log \rho_s(x|x_0) = -\frac{1}{\sqrt{s}\sigma_s }\mathbb{E}^{x_0}[z|x_s=x]
\end{equation}
where $x_s = \alpha_s x_0 + \beta_s x_1 + \sqrt{s} \sigma_s z$ and where we used $x_s|x_0 \stackrel{d}{=} I_s|x_0 \stackrel{d}{=} X_s$ for all $s\in[0,1]$.
Since  
\begin{equation}
\begin{aligned}
    b_s(x,x_0) &= \dot{\alpha}_s x_0 + \dot{\beta}_s\mathbb{E}^{x_0}[x_1|x_s =x] + \sqrt{s} \dot{\sigma}_s\mathbb{E}^{x_0}[z|x_s =x]\\
    x &= \alpha_s x_0 + \beta_s \mathbb{E}^{x_0}[x_1|x_s =x] + \sqrt{s} \sigma_s \mathbb{E}^{x_0}[z|x_s =x]
\end{aligned}
\end{equation}
we deduce that
\begin{equation}
    \mathbb{E}^{x_0}[z|x_s = x] = \frac{\beta_s b_s(x,x_0) -\dot{\beta}_sx - (\beta_s\dot{\alpha}_s-\dot{\beta}_s\alpha_s)x_0}{\sqrt{s}(\dot{\sigma}_s\beta_s-\sigma_s\dot{\beta}_s)}
\end{equation}
which can be inserted in~\eqref{eq:stein} to show that~\eqref{eq:score:app} holds.

Second, note that the PDF $\rho_s(x|x_0)$ of the solution to the SDE~\eqref{eq:sde-appenfix} is the same as the PDF of the solution to
\begin{equation}
    dX_s= b_s(X_s,x_0) ds + \sigma_s  dW_s, \quad X_{s=0} = x_0,
\end{equation}
since we can use the identity $\tfrac12 \sigma^2_s \Delta \rho_s = \tfrac12 g^2_s \Delta \rho_s -\tfrac12 (g_s^2-\sigma_s^2) \nabla \cdot (\rho_s \nabla \log \rho_s)$ to show that both densities satisfy the same Fokker-Planck equation~\eqref{eq:FPE}. 

It remains to show that the drift coefficient~\eqref{eq:drift:bc-appendix:a} is well-defined. To this end let us  write it explicitly as
\begin{equation}
\label{eq:bgs:full}
\begin{aligned}
    b^g_s(x,x_0) &= b_s(x,x_0) + \left(\frac{g_s^2-\sigma_s^2}{2s\sigma_s}\right)\left( \frac{\beta_s b_s(x,x_0) -\dot{\beta}_sx - (\beta_s\dot{\alpha}_s-\dot{\beta}_s\alpha_s)x_0}{\sigma_s\dot{\beta}_s-\dot{\sigma}_s\beta_s}\right).
\end{aligned}
\end{equation}
Our assumptions that $\alpha,\beta,\sigma\in C^1([0,1])$ and satisfy $\alpha_0=\beta_1 =1$, $\alpha_1=\beta_0=\sigma_1=0$, and $\dot\beta_s >0$ for all $s\in(0,1]$ and $\dot \sigma_s<0$ for all $s\in[0,1]$ guarantee that $\beta_s>0$ for all $s\in(0,1]$ and $\sigma_s>0$ for all $s\in[0,1)$ and, as a result,  $[s\sigma_s (\dot\beta_s \sigma_s -\beta_s  \dot\sigma_s)]^{-1}$ is positive and finite for all $s\in(0,1)$. These assumptions also guarantee that $\beta_s  [\sigma_s\dot{\beta}_s-\dot{\sigma}_s\beta_s]^{-1}$ and $\dot\beta_s [\sigma_s\dot{\beta}_s-\dot{\sigma}_s\beta_s]^{-1}$ have finite limits at $s=0,1$.
Therefore the only factor in~\eqref{eq:bgs:full} that can be singular is $(g_s^2-\sigma_s^2)/(2s\sigma_s)$ at $s=0$, because of the factor $s^{-1}$,  and at $s=1$ because of the factor $\sigma^{-1}_1$. These singularities disappear under our assumptions that $\lim_{s\to 0^+} s^{-1}[g^2_s-\sigma^2_s]$ and $\lim_{s\to 1^-} g^2_s\sigma^{-1}_s$ exist and are finite. Therefore the drift $b^g_s$ has the same regularity properties has $b_s$.
\end{proof}

\subsection{Maximizing the likelihood with respect to the noise schedule}
\label{app:KL}

Let $X^g = (X^g_s)_{s\in[0,1]}$ be the process defined by the SDE~\eqref{eq:sde-appenfix}, and let $\hat X^g = (\hat X^g_s)_{s\in[0,1]}$ denote an approximate, learned process governed by
\begin{equation}
    \label{eq:sde-appenfix:explicit:hat}
    d\hat X^g_s= \left(1+\tfrac12 \beta_s A_s (g_s^2-\sigma_s^2) \right) \hat b_s(\hat X^g_s,x_0) ds - A_s c_s(\hat X_s^g,x_0) ds + g_s  dW_s, \qquad \hat X^g_{s=0} = x_0,
\end{equation}
With a slight abuse of notation, let us denote by $D_{\text{KL}}(X^g||\hat X^g)$ the KL divergence of the path measure of $X^g$ from the path measure of $\hat X^g$.
By Girsanov's theorem, it is given by
\begin{equation}
\label{eq:kl:a1}
        D_{\text{KL}} (X^g||\hat X^g)  = \frac12 \int_0^1 g^{-2}_s |1+\tfrac12 \beta_s  A_s(g_s^2-\sigma_s^2)|^2 \E |b_s^g(X^g_s,x_0)-\hat b_s^g(X^g_s,x_0)|^2 ds,
\end{equation}
where $\E$ denotes an expectation over the law of $X^g$. Using $X_s^g \stackrel{d}{=} I_s|x_0$,~\eqref{eq:kl:a1} can also be written as
\begin{equation}
\label{eq:kl:a2}
\begin{aligned}
    D_{\text{KL}} (X^g||\hat X^g)  & = \frac12 \int_0^1 g_s^{-2}  |1+\tfrac12 \beta_s  A_s(g_s^2-\sigma_s^2)|^2 L_s ds,\\
    &= \frac12 \int_0^1 g_s^{-2} |1-\tfrac12 \beta_s  A_s\sigma^2_s +\tfrac12 \beta_s  A_s g_s^2|^2 L_s ds,
\end{aligned}
\end{equation}
where $A_s$ is given in~\eqref{eq:df:A:c} and  $L_s=\E^{x_0} \big[ |\hat b_s(I_s,x_0)- b_s(I_s,x_0)|^2 \big]$. 
We now view the KL divergence~\eqref{eq:kl:a2} as an objective function for $g$, and solve the optimization problem
\begin{equation}
    \label{eq:min:prob:kl}
    \min_g  D_{\text{KL}} (X^g||\hat X^g).
\end{equation} 
Since $L_s$ is independent of $g_s$, minimizing~\eqref{eq:kl:a2} amounts to minimizing 
\begin{equation}
    \label{eq:min:int:a}
    g_s^{-2} |1-\tfrac12 \beta_s  A_s\sigma^2_s +\tfrac12 \beta_s  A_s g_s^2|^2
\end{equation}
for all $s\in [0,1]$. Since $A_s>0$, \eqref{eq:min:int:a} is minimized at
\begin{equation}
    \label{eq:min:int:b}
    g_s^2 = \frac{\frac12 \beta_s  A_s\sigma^2_s -1}{\frac12 \beta_s  A_s} \qquad \text{when}\quad 1-\tfrac12 \beta_s  A_s\sigma^2_s \le 0 \qquad \text{with} \quad \text{\eqref{eq:min:int:a}} = 0,
\end{equation}
and at 
\begin{equation}
    \label{eq:min:int:c}
    g_s^2 = \frac{1-\frac12 \beta_s  A_s\sigma^2_s}{\frac12 \beta_s  A_s} \qquad \text{when} \quad 1-\tfrac12 \beta_s  A_s\sigma^2_s >0 \qquad \text{with} \quad \text{\eqref{eq:min:int:a}} = 2\beta_s A_s(1-\tfrac12 \beta_s  A_s\sigma^2_s) >0.
\end{equation}
Together, these relations show that the minimizer of~\eqref{eq:min:int:a} is $g_s= g_s^\Fo $ with 
\begin{equation}
\label{eq:Ds:app}
g^\Fo_s = \left|\frac{1-\frac12 \beta_s  A_s\sigma^2_s}{\frac12 \beta_s  A_s}\right|^{1/2} = \left| 2s\sigma_s(\beta_s^{-1}\dot\beta_s \sigma_s- \dot\sigma_s) -\sigma_s^2\right|^{1/2},
\end{equation}
where we used the expression for $A_s$  in~\eqref{eq:df:A:c} to get the second equality. This result is \eqref{eq:g:spec}.  Notice that we can also write
\begin{equation}
\label{eq:Ds:2:app}
2s\sigma_s(\beta_s^{-1}\dot\beta_s \sigma_s- \dot\sigma_s) -\sigma_s^2 = 2s\sigma^2_s \frac{d}{ds}\log\frac{\beta_s }{ \sqrt{s}\sigma_s }.
\end{equation}
Since the sign of this expression is the same as the sign of $1-\tfrac12 \beta_s  A_s\sigma^2_s$, it shows that \eqref{eq:min:int:b} and \eqref{eq:min:int:c} arise at values of $s\in[0,1]$ where $\beta_s /[\sqrt{s}\sigma_s]$ is respectively decreasing or increasing in $s$. Since $\dot\beta_s >0$ and $\dot\sigma_s<0$, $\beta_s /[\sqrt{s}\sigma_s]$ cannot be decreasing for all $s\in[0,1]$, i.e. the minimum of \eqref{eq:min:int:a} cannot be zero for all $s\in[0,1]$. As a result, the minimum of the KL divergence \eqref{eq:kl:a2} must be positive if $L_s>0$.

Finally, let us investigate the conditions on $g_s$ in Theorem~\ref{thm:tunning_g} if we use the SDE~\eqref{eq:sde-appenfix:explicit:hat} with $g_s=g_s^\Fo$. It is easy to see from~\eqref{eq:Ds:app} that the second condition is always satisfied since
\begin{equation}
    \label{eq:gF:s1}
    \lim_{s\to1^-} |g^\Fo_s|^2\sigma_s^{-1} = \lim_{s\to1^-} [2s (\beta_s^{-1} \dot\beta_s \sigma_s -\dot\sigma_s) -\sigma_s] = -2\dot\sigma_1 >0
\end{equation}
since $\sigma_1=0$. Regarding the first condition, we have
\begin{equation}
    \label{eq:gF:s0}
    \lim_{s\to0^+} s^{-1} [ |g^\Fo_s|^2-\sigma^2_s] = \lim_{s\to0^+} [2\sigma_s (\beta_s^{-1} \dot\beta_s \sigma_s -\dot\sigma_s) -2s^{-1}\sigma^2_s] = 2\lim_{s\to0^+} [\beta_s^{-1} \dot\beta_s -s^{-1}]-2\dot\sigma_0\sigma_0 
\end{equation}
since $\sigma_0>0$, $\dot\sigma_0<0$. Since $\beta_0=0$, if $\dot\beta_0>0$ and $\ddot\beta_0$ exists, we have
\begin{equation}
    \label{eq:lim}
    \lim_{s\to0^+} [\beta_s^{-1} \dot\beta_s -s^{-1}] = \tfrac12\ddot\beta_0 \dot\beta_0^{-1},
\end{equation}
so that 
\begin{equation}
    \label{eq:gF:s0a}
    \lim_{s\to0^+} s^{-1} [ |g^\Fo_s|^2-\sigma^2_s] =\sigma_0^2 \ddot\beta_0 \dot\beta_0^{-1}-2\dot\sigma_0\sigma_0 \qquad \text{if $\dot\beta_0>0$ and $\ddot\beta_0$ exists.}
\end{equation}
If, however $\dot \beta_0=0$, then $g_{0}^\Fo\not = \sigma_0$, and  the first condition is not satisfied since $\lim_{s\to0^+} s^{-1} [ |g^\Fo_s|^2-\sigma^2_s]$ does not exist. In this case, we need to consider more carefully how to define the solution to the SDE~\eqref{eq:sde-appenfix:explicit:hat}. In the specific case when $\alpha_s = 1-s$, $\sigma_s= \varepsilon (1-s)$ and $\beta_s = s^2$, if we set $g_s = g^\Fo_s=\varepsilon\sqrt{(1-s)(3-s)}$ in~\eqref{eq:sde-appenfix:explicit:hat}, this SDE reduces to (see the explicit formulas given in Appendix~\ref{app:specifics} and denoting $X^\Fo_s=X^{g^\Fo}_s$):
\begin{equation}
    \label{eq:sde:specific:Fol}
    d X^\Fo_s = \Big(1+\frac{1}{2-s}\Big)  b_s(X^\Fo_s,x_0)ds  - \frac{1}{s(2-s)}(2X^\Fo_s-(2-s)x_0) ds + \varepsilon\sqrt{(1-s)(3-s)} dW_s, \qquad X^\Fo_{s=0} = x_0.
\end{equation}
The drift in this equation is singular at $s=0$ because of the term $2(X^\Fo_s-x_0)/[s(2-s)]$. Nevertheless, the solution to this SDE is well-defined for the initial condition $X^\Fo_{s=0} = x_0$, and satisfies the integral equation
\begin{equation}
    \label{eq:sde:specific:Fol:int}
    X^\Fo_s = x_0 + \frac{2-s}{s} \int_0^s \frac{u}{2-u} \left(\Big(1+\frac{1}{2-u}\Big)  b_u(X^\Fo_u,x_0)- \frac{1}{(2-u)}x_0\right)  du +  \varepsilon\frac{2-s}{s} \int_0^s \frac{u\sqrt{(1-u)(3-u)}}{2-u}  dW_u.
\end{equation}
Since $b_0(x_0,x_0) = -x_0$ when $\dot\beta_0=0$, this equation implies that 
\begin{equation}
    \label{eq:sde:specific:Fol:int:s}
    X^\Fo_s \stackrel{d}{=} x_0 (1-s) +  \eps W_s +o(s).
\end{equation}
which is also the law of $I_s|x_0$ as it should.

\subsection{Connection with F\"ollmer Processes and Proof of Theorem~\ref{th:interp:f:si}}
\label{app:foll}

To begin, we give some background on the F\"ollmer process~\cite{follmer1986time,tzen2019theoretical}. 
Originally, the F\"ollmer process was defined as the process $X = (X_s)_{s\in[0,1]}$ whose path measure has minimal KL divergence from the path measure of the Wiener process $W= (W_s)_{s\in[0,1]}$ under the constraint that $X_{s=1}$ be distributed according to some target distribution. 
This F\"ollmer process can be generalized to ``reference processes'' that differ from the standard Wiener process: as we will show next, in our context the natural reference process is the solution to the linear SDE  
\begin{equation}
    \label{eq:ref:foll-gen}
    dY_s = a_s (Y_s - \alpha_s x_0)ds + \dot\alpha_s x_0 ds + g^\Fo_s dW_s, \qquad Y_{s=0} = x_0,
\end{equation}
where $g^\Fo_s$ is given by \eqref{eq:Ds:app} and where we have defined
\begin{equation}
\label{ea:as}
\begin{aligned}
    a_s & := \beta^{-1}_s\dot\beta_s - \frac{2s\sigma_s (\beta_s^{-1}\dot\beta_s \sigma_s - \dot \sigma_s) - \sigma_s^2}{\beta_s^2 + s \sigma^2_s},\\
    & = \beta^{-1}_s\dot\beta_s - \frac{2\beta^{-1}_s\dot\beta_s (\beta_s^2 + s\sigma_s^2) - 2\beta_s\dot\beta_s -2s\sigma_s\dot\sigma_s - \sigma_s^2}{\beta_s^2 + s \sigma^2_s},\\
    & = - \beta^{-1}_s\dot\beta_s  +  \frac{2\beta_s\dot\beta_s +2s\sigma_s\dot\sigma_s + \sigma_s^2}{\beta_s^2 + s \sigma^2_s} = \frac{d}{ds} \log \frac{\beta_s^2 + s\sigma_s^2}{\beta_s}.
\end{aligned}
\end{equation}
 We can then define  the F\"ollmer process $X^\Fo = (X_s^\Fo)_{s\in[0,1]}$ by adjusting the drift $\check b_s(x,x_0)$ in 
\begin{equation}
    \label{eq:ref:foll-gen:2}
    d\check X_s = \check b_s(\check X_s,x_0) ds + g^\Fo_s dW_s, \qquad \check X^\Fo_{s=0} = x_0,
\end{equation}
in such a way that the KL divergence of the path measure of $\check X$ from the path measure of $Y$ (the solution to~\eqref{eq:ref:foll-gen}) is minimized subject to the constraint that $\check X_{s=1}\sim \rho_c(\cdot|x_0)$. That is, the F\"ollmer process is defined via
\begin{equation}
    \label{eq:foll:min:prob}
    \min_{\check b} D_{\text{KL}} (\check X|| Y) \qquad \text{subject to} \quad   \check X_{s=1}\sim \rho_c(\cdot|x_0).
\end{equation}
This traditional minimization problem for the F\"ollmer process is distinct from the minimization problem~\eqref{eq:min:prob:kl} considered in Appendix~\ref{app:KL}. 
Nevertheless, our next result shows that the minimizers of~\eqref{eq:min:prob:kl} and \eqref{eq:foll:min:prob} coincide:

\begin{theorem}
    \label{th:foll:app}
    Assume that $\beta_s /[\sqrt{s}\sigma_s]$ is non-decreasing on $s\in[0,1]$.
    Then, the F\"ollmer process associated with the reference process $(Y_s)_{s\in[0, 1]}$ that solves~\eqref{eq:ref:foll-gen} is the process $X^\Fo \equiv (X^{g^{\rm F}}_s)_{s\in[0, 1]}$ that solves~\eqref{eq:sde-appenfix} with $g_s = g^\Fo_s$ given by~\eqref{eq:Ds:app}.
\end{theorem}
Theorem~\ref{th:interp:f:si} is implied by this result.

\begin{proof}
We closely follow the steps of F\"ollmer's original construction involving time-reversal~\cite{follmer1986time}.
To begin, notice that we can solve~\eqref{eq:score:app} to express $b_s$ in terms of the score $\nabla \log \rho_s$.
We can then use the resulting expression in~\eqref{eq:drift:bc-appendix:a} to write the SDE~\eqref{eq:sde-appenfix} as 
\begin{equation}
    \label{eq:sde-appenfix:2}
    dX^g_s= \left(\beta_s^{-1} A^{-1}_s +\tfrac12 (g_s^2-\sigma_s^2)\right)  \nabla \log \rho_s(X^g_s|x_0) ds + \beta_s^{-1} c_s(X^g_s,x_0) ds  + g_s  dW_s, \qquad X^g_{s=0} = x_0,
\end{equation}
where $A_s$ and $c_s(x,x_0)$ are defined in~\eqref{eq:df:A:c}.
Using the score, we can time-reverse~\eqref{eq:sde-appenfix:2} and derive the following SDE for $X_{s}^\rev \overset{d}{=} X^g_{1-s}$
\begin{equation}
    \label{eq:sde-appenfix:3}
    dX^\rev_s= -\left(\beta_{1-s}^{-1} A^{-1}_{1-s} -\tfrac12 (g_{1-s}^2+\sigma_{1-s}^2)\right)  \nabla \log \rho_{1-s}(X^\rev_s|x_0) ds - \beta_{1-s}^{-1} c_{1-s}(X^\rev_s,x_0) ds  + g_{1-s}  dW_s.
\end{equation}
If we take $X^\rev_{s=0} \sim \rho_c(\cdot|x_0)$ in~\eqref{eq:sde-appenfix:3}, then by construction we have that $X^\rev_{s=1} = x_0$. 
Remarkably, if we use $g_s=g^\Fo_s$ with $g_s^\Fo$ given in~\eqref{eq:Ds:app}, and  if we choose $1-\frac12 \beta_s  A_s\sigma^2_s\ge 0$ (i.e. choose $\beta_s /[\sqrt{s}\sigma_s]$ nondecreasing), the SDE~\eqref{eq:sde-appenfix:3} reduces to
\begin{equation}
    \label{eq:sde-appenfix-}
    dX^\rev_s= -\beta_{1-s}^{-1} c_{1-s}(X^\rev_s,x_0) ds  + g^\Fo_{1-s}  dW_s, \qquad X^\rev_{s=0}  \sim \rho_c(\cdot|x_0).
\end{equation}
This reverse-time SDE has the remarkable property that its drift is independent of the score $\nabla \log \rho_s$, meaning that the information about the target PDF $\rho_c(\cdot|x_0)$ only enters through its initial condition~$X^\rev_{s=0} \sim \rho_c(\cdot|x_0)$. 
This  means that we can change the initial condition in~\eqref{eq:sde-appenfix-} to construct (after reversing time back) a reference process.
We can use any density for this purpose, but for simplicity it is convenient to choose a Gaussian, and therefore to consider
\begin{equation}
    \label{eq:sde-appenfix:4}
    dY^\rev_s= -\beta_{1-s}^{-1} c_{1-s}(Y^\rev_s,x_0) ds  + g^\Fo_{1-s}  dW_s, \qquad Y^\rev_{s=0}  \sim {\sf N}(0,\Id).
\end{equation} 
Using the explicit form of $c_s$ in~\eqref{eq:df:A:c-appendix}, we can write this SDE explicitly as
\begin{equation}
    \label{eq:sde-appenfix:5}
    dY^\rev_s= -\beta_{1-s}^{-1}\dot\beta_{1-s} (Y^\rev_s -\alpha_{1-s} x_0) ds - \dot\alpha_{1-s} x_0 + g^\Fo_{1-s}  dW_s, \qquad Y^\rev_{s=0}  \sim {\sf N}(0,\Id).
\end{equation} 
Using that $d\alpha_{1-s} = - \dot\alpha_{1-s} ds$ and $d\beta_{1-s} = - \dot\beta_{1-s} ds$, we may rewrite this as
\begin{equation}
    \label{eq:sde-appenfix:6}
    d \big( \beta^{-1}_{1-s}( Y^\rev_s-\alpha_{1-s} x_0) \big)= \beta_{1-s}^{-1} g^\Fo_{1-s}  dW_s, \qquad Y^\rev_{s=0}  \sim {\sf N}(0,\Id).
\end{equation} 
Hence,
\begin{equation}
    \label{eq:sde-appenfix:7}
    Y^\rev_s = \alpha_{1-s} x_0 + \beta_{1-s} Y^\rev_{s=0} + \beta_{1-s} \int_0^s \beta^{-1}_{1-u}g^\Fo_{1-u}  dW_u.
\end{equation} 
Using the explicit form of $g^{\rm F}_s$ given in~\eqref{eq:Ds:app} together with \eqref{eq:Ds:2:app}, we deduce that
\begin{equation}
    \label{eq:sde-appenfix:8}
    \begin{aligned}
        \E \left[\left|\beta_{1-s} \int_0^s \beta^{-1}_{1-u}g^\Fo_{1-u}  dW_u\right|^2\right] & = \beta^2_{1-s} \int_0^s \beta^{-2}_{1-u}|g^\Fo_{1-u}|^2 du,\\
    &=\beta^2_{1-s} \int_{1-s}^1 \beta^{-2}_{u}|g^\Fo_{u}|^2
     du,\\
    & = \beta^2_{1-s} \int_{1-s}^1 \frac{2u\sigma^2_{u}}{\beta^2_{u}} \frac{d}{du} 
    \log \frac{\beta_{u}}{\sqrt{u} \sigma_{u}} du,\\
    & = -\beta^2_{1-s} \int_{1-s}^1 d\left(\frac{u\sigma^2_{u}}{\beta^2_{u}}\right),\\
    &= (1-s)\sigma^2_{1-s},
    \end{aligned}
\end{equation}
where we used $\sigma_1 = 0$.
This means that 
\begin{equation}
\label{eq:Y:ref}
Y^\rev_{1-s}\stackrel{d}{=} \alpha_s x_0 + \beta_s z + \sigma_s W_s \sim {\sf N}(\alpha_s x_0, \beta_s^2+ s\sigma_s^2),
\end{equation}
where $z\sim {\sf N}(0,\Id)$ with $z\perp W$.
Note that, unsurprisingly, the process on the right-hand side  is simply the stochastic interpolant~\eqref{eq:stoch:int:def} conditioned on $x_0$ fixed, with $x_1$ replaced by a Gaussian $z$.  Denoting by $\rho_s^{Y}(y|x_0)$ the PDF of $Y^\rev_{1-s}$, \eqref{eq:Y:ref} implies that
\begin{equation}
    \label{eq:score:ref}
    \nabla \log \rho_s^{Y}(y|x_0) = -\frac{y-\alpha_s x_0}{\beta_s^2+ s\sigma_s^2}.
\end{equation}
We can use this result  to time reverse~\eqref{eq:sde-appenfix:4} and obtain the following SDE for $Y_s \stackrel{d}{=} Y^\rev_{1-s}$
\begin{equation}
    \label{eq:Y:ref:sde}
    dY_s = \beta_{s}^{-1}\dot\beta_s (Y_s -\alpha_{s} x_0) ds + \dot\alpha_{s} x_0 + |g_s^\Fo|^2 \nabla \log \rho_s^{Y}(Y_s|x_0) ds+ g^\Fo_s  dW_s  \qquad Y_{s=0}  =x_0.
\end{equation}
If we insert the explicit form of $|g^{\rm F}_s|^2 = 2s\sigma_s (\beta_s^{-1}\dot\beta_s \sigma_s - \dot \sigma_s) - \sigma_s^2$ given in~\eqref{eq:Ds:app} into~\eqref{eq:Y:ref:sde} we get
\begin{equation}
    dY_s = a_s (Y_s - \alpha_s x_0)ds + \dot\alpha_s x_0 ds + g^\Fo_s dW_s, \qquad Y_{s=0} = x_0,
\end{equation}
with
\begin{equation}
    a_s = \beta^{-1}_s\dot\beta_s - \frac{2s\sigma_s (\beta_s^{-1}\dot\beta_s \sigma_s - \dot \sigma_s) - \sigma_s^2}{\beta_s^2 + s \sigma^2_s},
\end{equation}
i.e. we recover the SDE~\eqref{eq:ref:foll-gen}. Since $a_s = \frac{d}{ds} \log \frac{\beta_s^2 + s\sigma_s^2}{\beta_s}$, we can solve the above the SDE to obtain the reference process as
\[Y_s = \alpha_s x_0 + \frac{\beta^2_s+s\sigma^2_s}{\beta_s}\int_0^s \frac{\beta_ug^{\rm F}_u}{\beta^2_u+u\sigma^2_u}dW_u.\]
We can then verify that 
\begin{equation}
    \begin{aligned}
        \E\left[\left|\frac{\beta^2_s+s\sigma^2_s}{\beta_s}\int_0^s \frac{\beta_ug^{\rm F}_u}{\beta^2_u+u\sigma^2_u}dW_u\right|^2\right] &= (\frac{\beta^2_s+s\sigma^2_s}{\beta_s})^2 \int_0^s \frac{\beta_u^2\sigma_u^2 u}{(\beta_u^2 + \sigma_u^2 u)^2} \frac{d}{du}\log\frac{\beta_u^2}{\sigma_u^2 u} du \\
        & = (\frac{\beta^2_s+s\sigma^2_s}{\beta_s})^2 \int_0^s \frac{1}{((\beta_u^2/\sigma_u^2 u) + 1)^2} d(\beta_u^2/\sigma_u^2 u) \\
        & = (\frac{\beta^2_s+s\sigma^2_s}{\beta_s})^2 \frac{\beta_s^2}{\beta^2_s+s\sigma^2_s} = \beta^2_s+s\sigma^2_s,
    \end{aligned}
\end{equation}
which implies that $Y_s \stackrel{d}{=} \alpha_s x_0 + \beta_s z + \sigma_s W_s \sim {\sf N}(\alpha_s x_0, \beta_s^2+ s\sigma_s^2)$; this matches our previous calculation \eqref{eq:Y:ref}.

It remains to show that the process $X^\Fo \equiv X^{g^{\rm F}}$ defined by solution to the SDE~\eqref{eq:sde-appenfix} with $g_s = g^\Fo_s$ given in~\eqref{eq:Ds:app} is the F\"ollmer process associated with the process $Y$ defined by the solution to~\eqref{eq:Y:ref:sde}. %
To this end, recall that the KL divergence between two path measures is invariant under time-reversal, so that
\begin{equation}
\label{eq:KL:trev}
D_{\text{KL}} (\check X|| Y) = D_{\text {KL}} (\check X^\rev|| Y^\rev)
\end{equation}
where $\check X$ is the process defined by the solution to the SDE~\eqref{eq:ref:foll-gen:2}, and where $\check X^\rev$ is its time-reversal.
We can now use following decomposition, known as ``disintegration''~\cite{leonard2014survey}, of the Kullback-Leibler divergence 
\begin{equation}
\label{eq:kl:foll:gen:1}
         D_{\text {KL}} (\check X^{\rev}|| Y^{\rev}) =  \int_{\R^d} D_{\text {KL}} (\check X^{\rev,x}|| Y^{\rev,x}) \rho_c(x|x_0) dx + D_{\text{KL}}[\rho_c(\cdot |x_0) \Vert {\sf N}(x_0,\Id)]
\end{equation}
where $\check X^{\rev,x}$ and $ Y^{\rev,x}$ denote, respectively, the processes $\check X^\rev$ and $Y^\rev$ conditioned to start from $x$ (i.e. on $\check X^{\rev,x}_{s=0}= Y^{\rev,x}_{s=0}=x$). 
In addition, we used the fact that $Y^\rev_{s=0}\sim {\sf N}(0,\Id)$ whereas $\check X^\rev_{s=0}\sim \rho_c(\cdot|x_0)$ by the constraint imposed in the minimization problem~\eqref{eq:foll:min:prob}. The second term at the right hand side of~\eqref{eq:kl:foll:gen:1} is fixed due to this constraint; the first term is always non-negative, but we can make it zero if we take $\check X^\rev=X^\rev$ with $X^\rev$ defined as the solution to~\eqref{eq:sde-appenfix:3} since this process is the same as the one defined by the solution to~\eqref{eq:sde-appenfix:4} if we condition both on $X^\rev_{s=0} = Y^\rev_{s=0} = x$. Therefore $\check X^\rev=X^\rev$ minimizes \eqref{eq:kl:foll:gen:1}, which means that its time-reversal $X^{g^{\rm F}}$ minimizes~\eqref{eq:foll:min:prob}.
\end{proof}

\subsection{Analytic formula of $b_s$ for Gaussian mixture distributions}
When $\rho_c(x|x_0)$ is a Gaussian mixture model (GMM), the drift $b_s$ is available analytically:
\begin{proposition}
   Let the target density be a GMM with $J\in \N$ modes
\begin{equation}
\label{eq:gmm}
\rho_c(x|x_0) = \sum_{j=1}^J p_j {\sf N}(x;m_j,C_j)
\end{equation}
where $p_j\ge0$ with $\sum_{j=1}^J p_j =1$, $m_j \in \R^d$, and $C_j = C_j^T \in \R^d \times \R^d$ positive-definite (with both $m_j$ and $C_j$ possibly dependent on $x_0$). Then 
\begin{equation}
\label{eq:b}
    \begin{aligned}
    b_s(x,x_0) &= \dot\alpha_s x_0 + \dot\beta_s \frac{\sum_{j=1}^J p_j m_j {\sf N}(x;\overline{m}_j(s),\overline{C}_j(s))}{\sum_{j=1}^J p_j {\sf N}(x;\overline{m}_j(s),\overline{C}_j(s))} \\
 &+ \frac{\sum_{j=1}^J p_j (\beta_s\dot\beta_s C_j + s\sigma_s\dot\sigma_s \Id) \overline{C}_j^{-1}(s) (x-\overline{m}_j(s)) {\sf N}(x; \overline{m}_j(s),\overline{C}_j(s))}
 {\sum_{j=1}^J p_j {\sf N}(x;\overline{m}_j(s),\overline{C}_j(s))}
 \end{aligned}
\end{equation}
where
\begin{equation}
    \label{eq:m:C}
    \begin{aligned}
        \overline{m}_j(s) = \alpha_s x_0 + \beta_s m_j, \qquad 
        \overline{C}_j(s) = \beta^2_s C_j + s \sigma^2_s \Id .
    \end{aligned}
\end{equation}
\end{proposition}
\begin{proof}
By definition
    \begin{equation}
    \label{eq:b:gmm:a}
\begin{aligned}
    b_s(x,x_0) & = \dot\alpha_s x_0 + \E[\dot\beta_s x_1 + \dot\sigma_s W_s|I_s=x,x_0] \\
    &= \dot\alpha_s x_0 + \E[\dot\beta_s\beta_s^{-1}( x-\alpha_s x_0-\sigma_s W_s) + \dot\sigma_s W_s|I_s=x,x_0]\\
    & = \dot\alpha_sx_0 + \dot\beta_s\beta^{-1}_s(x-\alpha_sx_0) + s\sigma_s(\sigma_s\dot\beta_s\beta_s^{-1}-\dot\sigma_s)\nabla \log \rho_s(x|x_0).
\end{aligned}
\end{equation}
where we used the fact $\nabla \log \rho_s(x|x_0) = - [s\sigma_s]^{-1}\E^{x_0}[W_s|I_s=x]$. For the GMM, 
\begin{equation}
    \rho_s(x|x_0) = \sum_{j=1}^J p_j {\sf N}(x; \overline{m}_j(s), \overline{C}_j(s)),
\end{equation}
so that 
\begin{equation}
    \nabla \log \rho_s(x|x_0) = -\frac{\sum_{j=1}^J p_j \overline{C}_j^{-1}(s)(x-\overline{m}_j(s)) {\sf N}(x;\overline{m}_j(s),\overline{C}_j(s))}{\sum_{j=1}^J p_j {\sf N}(x;\overline{m}_j(s),\overline{C}_j(s))}.
\end{equation}
Inserting this expression in \eqref{eq:b:gmm:a} we obtain
\begin{equation}
    \begin{aligned}
        &\frac{\dot\beta_s}{\beta_s}(x-\alpha_s x_0) + s\sigma_s^2 \frac{\dot\beta_s}{\beta_s}\nabla \log \rho_s(x|x_0)\\
        =& \frac{\dot\beta_s}{\beta_s}\left(x-\alpha_s x_0 - \frac{\sum_{j=1}^J p_j (\Id - \beta^2_sC_j\overline{C}_j^{-1}(s))(x-\overline{m}_j(s)) {\sf N}(x;\overline{m}_j(s),\overline{C}_j(s))}{\sum_{j=1}^J p_j {\sf N}(x;\overline{m}_j(s),\overline{C}_j(s))}\right)\\
        =&  \frac{\dot\beta_s}{\beta_s}\left(\frac{\sum_{j=1}^J p_j \big(\beta_s m_j +  \beta^2_s C_j\overline{C}_j^{-1}(x-\overline{m}_j)\big) {\sf N}(x;\overline{m}_j(s),\overline{C}_j(s))}{\sum_{j=1}^J p_j {\sf N}(x;\overline{m}_j(s),\overline{C}_j(s))}\right)\\
        =&  \dot\beta_s \frac{\sum_{j=1}^J p_j m_j {\sf N}(x;\overline{m}_j(s),\overline{C}_j(s))}{\sum_{j=1}^J p_j {\sf N}(x;\overline{m}_j(s),\overline{C}_j(s))} + \frac{\sum_{j=1}^J p_j \beta_s \dot\beta_s C_j\overline{C}_j^{-1}(x-\bar{m}_j){\sf N}(x;\overline{m}_j(s),\overline{C}_j(s))}{\sum_{j=1}^J p_j {\sf N}(x;\overline{m}_j(s),\overline{C}_j(s))},
    \end{aligned}
\end{equation}
where in the first and second identities, we used the fact that $s\sigma^2_s\Id = \Id - \beta^2_sC_j\overline{C}_j^{-1}(s)$ and $x - \alpha_s x_0 = x - \overline{m}_j(s) + \beta_s m_j$.

Now, using $b_s(x,x_0) = \dot\alpha_s x_0 + \dot\beta_s\beta_s^{-1}(x-\alpha_s x_0) + s\sigma_s^2(\dot\beta_s\beta_s^{-1}-\dot\sigma_s)\nabla \log \rho_s(x|x_0)$, we get the final formula.
\end{proof}

\section{Details of Numerical Experiments}
\label{sec:detail:num}
\subsection{Multi-modal jump-diffusion process}
\label{sec:detail:jump:diff}

Inspired by \cite{chen2020neural}, we create the synthetic example by starting with a Gaussian component ${\sf N}(m_0,C_0)$ with $m_0 = [5,0]$ and
\[C_0 = \begin{bmatrix}
1.5 & 0 \\
0 & 0.1
\end{bmatrix}.\]
We rotate this distribution counterclockwise by $2\pi/5$ four times to obtain the remaining four Gaussian modes. We assign each of the five modes equal weights to obtain our 2D Gaussian mixture model with density $\rho_{\text{GMM}}(x)$

The 2D particle jump-diffusion dynamics is constructed as follows. Between the jumps, the particle moves according to the Langevin dynamics 
$$dx_t = \nabla \log \rho_{\text{GMM}}(x_t) dt + \sqrt{2}dW_t.$$
At jump times specified by a Poisson process with rate $\lambda = 2$, the particle is rotated counterclockwise by an angle $2\pi/5$. 

We simulate this dynamics using the following scheme with $\delta t=0.01$:
\begin{align*}
    x_{(n+1)\delta t} =
    \begin{cases}
      x_{n\delta t} + \delta t \nabla \log p(x_{n\delta t}) + \sqrt{2\delta t}\xi & \text{with probability $1-\lambda \delta t$}\\
      R_{2\pi/5}(x_{n\delta t} + \delta t \nabla \log p(x_{n\delta t}) + \sqrt{2\delta t}\xi) & \text{with probability $\lambda \delta t$}
    \end{cases}    
\end{align*}
where $\xi \sim {\sf N}(0,I_{2\times 2})$ and $R_{2\pi/5}$ is the counterclockwise rotation operator in 2D with angle $2\pi/5$. We integrate this dynamics long enough to reach equilibrium and get enough data.

We keep the data at a regular time interval of $\Delta t = 0.5$ and use paired $(x_t, x_{t+0.5})$ as training data for learning the conditional distribution at lag $\tau = 0.5$. In total we store $10^5$ training  data pairs. We use a fully connected neural network with $5$ hidden layes with hidden dimension $500$ to approximate the velocity field $b_s(x,x_0)$ in the SDEs. The input to the net is of dimension $5$ and the output is of dimension $2$. We train the network using a batch size of $10^4$, default AdamW optimizer with base learning rate $l = 10^{-3}$ and cosine scheduler that decreases in each epoch the learning rate eventually to $0$ after $300$ epochs. We test the SDEs for new simulated trajectory data.
\subsection{2D Stochastic Navier-Stokes Example}
\label{sec:detail:nse}

We set $\nu = 10^{-3}, \alpha = 0.1, \varepsilon = 1$.
We consider the following random forcing
\begin{equation}
\label{eq:noise_construction}
\begin{aligned}
    \eta(t,x,y) = &\big( W_1 (t)  \sin (6x) +  W_2 (t) \cos (7x) + W_3 (t) \sin (5(x+y)) + W_4 (t) \cos (8(x+y)) \big) \\
    & + \big( W_5 (t)  \cos (6x) +  W_6 (t) \sin (7x) + W_7 (t) \cos (5(x+y)) + W_8 (t) \sin (8(x+y)) \big),
\end{aligned}
\end{equation}
where $W_i(t)$, $1\leq i \leq 8$ are independent Wiener processes. For such random forcing, the NSE~\eqref{eq:2D_vorticity_NS} has a unique invariant measure, as proved in \citet{hairer2006ergodicity}. Note that \eqref{eq:noise_construction} is a Gaussian random field with covariance function $C(t,t',x,y, x',y') = \min(t,t')[\cos(6(x-x')) + \cos(7(x-x')) + \cos(5(x-x'+y-y') + \cos(8(x-x'+y-y'))]$ which is translation invariant in space. The damping term $- \alpha \omega$ is used to accelerate the mixing of the dynamics as it damps the vorticity at large length-scales/small Fourier modes to avoid all energy accumulating in large vortices. 

We use a pseudo-spectral solver with Euler-Maruyama time-stepping scheme to solve the stochastic PDE in time and we use the jax-cfd package \cite{Dresdner2022-Spectral-ML} for the mesh generation and domain discretization. We perform simulations with a grid sizes $256\times 256$ which is fine enough for resolving the numerical solutions of the 2D stochastic Navier-Stokes equations with our specifications. 

We notice that there are existing works on forecasting with generative models that experiment with the 2D NSE \cite{Lienen2023FromZT} \cite{cachay2023dyffusion}; however, they are interested in deterministic forecasting. In contrast, we are interested in probabilistic forecasting which is why we add a stochastic forcing to the equation. The precise nature of this forcing is probably not important in practice.
In all of our experiments, we normalize the training and testing data so that on average the $L^2$ norm of each snapshot field in the training data is $1$; we found such normalization is useful to ensure a stable gradient that does not blow up during training. We use a standard UNet architecture popularized by \citet{ho2020denoising} to learn the drift velocity $b_\theta$; the conditioning in the vector field is achieved by channel concatenation in the input, as this is the usual way to incorporate high dimensional conditioning in UNets for Diffusion \cite{ho2022cascaded}. The total parameters for the UNet is around $2$ millions. We employ the default AdamW optimizer \cite{loshchilov2017decoupled}. We use a cosine annealing schedule to decrease the learning rate in each epoch. The base learning rate is $l = 10^{-3}$. We split the data into $90\%$ training data and $10\%$ test data. We use a batch size of $100$. In total we train $50$ epochs. The model is trained on a single Nvidia A100 GPU and it takes less than $1$ day.
Once trained, we test our SDEs model (e.g. see Fig.~\ref{fig:cond-stats-NSE}) by examining the conditional distributions. We fix $50$ initial conditions drawn from the invariant measure and run simulations or our SDE models to the desired lags, and collect an ensemble of solutions (we choose the number of ensembles to be $300$). We compute mean and std of these ensembles and compare them with those obtained by using the simulation data. We also consider the enstrophy spectrum of the conditional distribution, which measures the size of the vorticity at each scale and reveals important physical information. More precisely, we compute the enstrophy of the vorticity field $\omega$ at wavenumber $k \in \mathbb{R}_+$ is defined as 
\[{\rm Enstrophy}(k) = \sum_{k\leq |m|\leq k+1} |\hat{\omega}(m)|^2,\]
where $\hat{\omega}$ with $m \in \mathbb{Z}^2$ denotes the Fourier coefficients of $\omega$. The enstrophy spectrum curve is the graph of the function $k \to {\rm Enstrophy}(k)$. We can average the enstrophy over different simulation trajectories of $\omega$ to get the spectrum of the distribution of the field.
We often smooth the curve to get rid of the sawtooth. 

We provide some additional numerical results on the 2D stochastic Navier-Stokes example below.
\paragraph{Superresolution} We consider superresolution, i.e., predicting the field with a resolution of $128\times 128$, from the downsized version with a resolution of $32\times 32$. In Fig.~\ref{fig:NSE_auto_data}, we show a $32\times 32$ resolution field, as well as the true $128\times 128$ resolution field and the mean of samples drawn from our SDEs. Our method achieves an outstanding recovery. Additionally, we calculate the standard deviation of the samples. The spatial distribution of this standard deviation serves as a tool for uncertainty quantification; notably, there exists a pronounced correlation between the pattern of standard deviation and the vorticity field.

\begin{figure}[ht]
    \begin{minipage}{.24\textwidth}
  \begin{subfigure}
    \centering
    \begin{overpic}[width=\linewidth]{media/image3/fig_supres_32.pdf}\put(33,-4){ $32\times 32$}\end{overpic}
  \end{subfigure}\\
  \begin{subfigure}
    \centering
    \begin{overpic}[width=\linewidth]{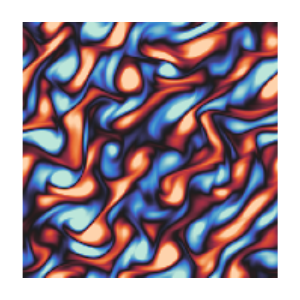}\put(30,-4){Sample mean}\end{overpic}
  \end{subfigure}
\end{minipage}
\begin{minipage}{.24\textwidth}
  \begin{subfigure}
    \centering
    \begin{overpic}[width=\linewidth]{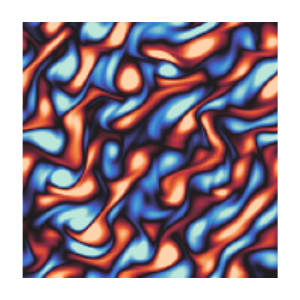}\put(32,-4){$128\times 128$}\end{overpic}
  \end{subfigure}\\
  \begin{subfigure}
    \centering
    \begin{overpic}[width=\linewidth]{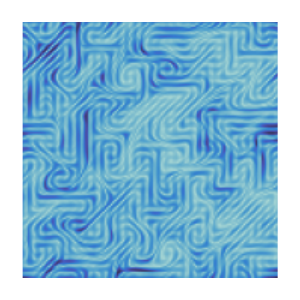}\put(30,-4){Sample std}\end{overpic}
  \end{subfigure}
\end{minipage}%
\begin{minipage}{.5\textwidth}
  \begin{subfigure}
    \centering
    \begin{overpic}[width=\linewidth]{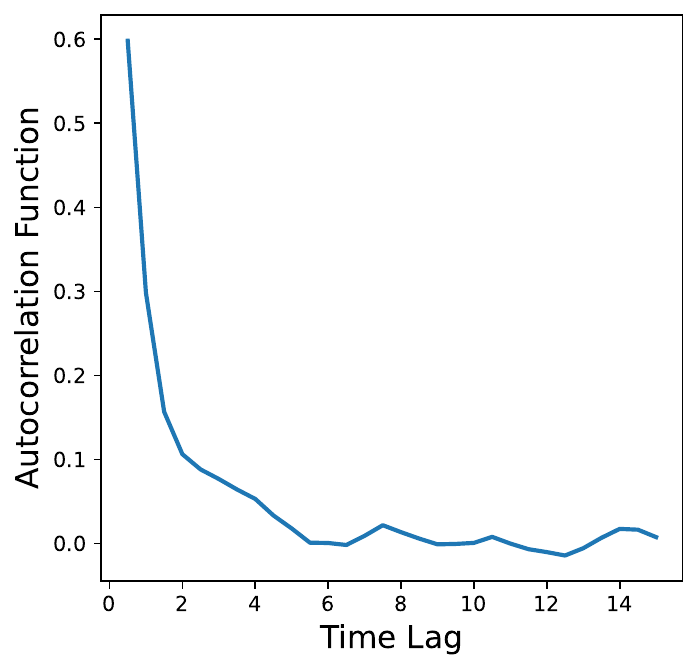}\end{overpic}
  \end{subfigure}
  \end{minipage}
  \vspace{-1.2em}
  \label{fig:NSE_auto_data}
    \caption{{\bf Left panel:} Superresolution. Low-resolution and high-resolution field, mean and std of samples drawn from our learned SDE model. {\bf Right panel:} We calculate the auto-correlation of our data as a function of time. The auto-correlation function is computed via averaging over the auto-correlation on all the grid points. Note that this averaging can be done because our random forcing term is spatially homogeneous and each grid point is identical in a statistical sense. We notice that the auto-correlation drops to about $0.25$ at a time lag $\tau = 1$ and about $0.1$ at a time lag $\tau = 2$, which indicates the potential diversity of the conditional distribution and thus the difficulty of our forecasting task.}
    \end{figure}

\paragraph{Forecasting efficiency}
We conducted a comparative analysis of prediction speed between our models and the SPDE solver utilized for the 2D stochastic Navier-Stokes equation. For a scenario involving a short time lag of $\tau = 0.5$, wherein our model accurately captures the resulting conditional distribution, sampling from our SDE forecaster with $200$ Euler-Maruyama  steps takes $0.05$ seconds, while executing the SPDE solver on the same Nvidia RTX8000 GPU requires $8$ seconds. This observation highlights that our method accelerates forecasting by over 100 times without sacrificing  physical information inherent in the conditional distribution. Furthermore, our approach could stand to gain even more acceleration in scenarios with lower viscosity, which typically necessitates a substantially reduced time step size for the SPDE solver.

\paragraph{Comparison with a deterministic approach}
We also contrast our probabilistic forecasting method with a deterministic approach. In the deterministic approach, we utilize the same UNet architecture to parameterize a point estimator, denoted as $\widehat{\omega}_{t + \tau} = f(\omega_t;\theta)$, and train it using the Mean Squared Error (MSE) loss function $\mathbb{E}\| \widehat{\omega}_{t + \tau} - \omega_{t + \tau}\|^2$, where $\tau = 0.5$ represents the time lag we aim to predict. This Minimum Mean Square Error (MMSE) estimator is designed to capture the mean of the conditional probability distribution. While the deterministic point estimator performs adequately in recovering the mean, it falls short in long-term total enstrophy prediction due to its inability to account for fluctuations in variance; see Table~\ref{table-accuracy-vs-map}. More precisely, we report the short time (lag = $0.5$) error of the conditional mean prediction: for the probabilistic approach we compute the mean of generated samples and compare it with the true conditional mean, and for the deterministic approach we use the point estimator to compare. For the probabilistic approach, we can also compute the predicted conditional std using samples and compare with the true std. We note that such comparison is regarding the relative $L^2$ err of the field of conditional mean and std. More precisely,
\[\mathrm{err(mean)} = \frac{\|\hat{m} - m\|_{L^2}}{\|m\|_{L^2}},\quad \mathrm{err(std)} = \frac{\|\hat{\sigma} - \sigma\|_{L^2}}{\|\sigma\|_{L^2}}, \]
where $\hat{m}, \hat{\sigma}$ are estimated mean and std (in this numerical example, it is an image of size $128\times 128$) and $m,\sigma$ are the true mean and std computed from simulation data.

These short time metrics characterize the accuracy of the forecasting in one lag time. We can also iterate the learned SDE or deterministic maps for many steps and use the generated dynamics to make predictions on the invariant distribution of the true dynamics.  We do so by iterating $100$ steps and compute the averaged total enstrophy (which is the squared $L^2$ norm of the vorticity field). We compare the predicted total enstrophy of the invariant distribution with the truth, in relative error. Here, the averaged total enstrophy is a scalar quantity; in contrast to the conditional mean and std. We obtain much better estimate of averaged total enstrophy with our SDEs. The result demonstrates the necessity of using probabilistic forecasting for stochastic dynamical systems.

\begin{table}[ht]
\centering
\begin{tabular}{ccc}
\toprule
 Relative error & Our SDEs & Deterministic  \\ \midrule
 Short time $L^2$ error of mean & \textbf{1.1e-1} & 2.3e-1 \\
 Short time $L^2$ error of std & 6.8e-2 & N/A \\
 Long time error of averaged total enstrophy & \textbf{5.6e-3} &  3.0e-1 \\ \bottomrule
\end{tabular}
\caption{Short-time and long-time relative accuracy comparisons between the SDE approach and the deterministic approach. The short-time accuracy is measured by the relative $L^2$ norm between the true conditional mean and the conditional mean of the samples in our SDE approach or the single output given by the deterministic network, at a lag $\tau = 0.5$. For SDEs, we also have the prediction of the conditional std. For the long-time comparisons, we iterate the SDE and the deterministic map both for $100$ steps and use the trajectories to estimate the averaged total enstrophy of the invariant distribution and compare them against the truth computed using samples from the true simulation data.}
\label{table-accuracy-vs-map}
\end{table}

\paragraph{Comparisons between different SDE generative models} All the above experiments on NSE is done with the interpolant $\alpha_s =1-s, \beta_s =s^2, \sigma_s =1-s$, as it performs the best. Below in Figures \ref{fig:loss_follmer_interpolants} and \ref{fig:spectrum_interpolants}, we also post the loss and gradient norm curves, as well as the enstrophy spectra of the generated samples, for other choices of interpolants and F\"ollmer processes (trained with the same network, data, and number of epochs). We observe that $\beta_s = s^2$, namely $\dot\beta_0 = 0$, is important to ensure a stable gradient norm curve. This experiment demonstrates the superiority of the interpolant $\alpha_s =1-s, \beta_s =s^2,\sigma_s =1-s$. Moreover, by changing the diffusion coefficients from $\sigma_s dW_s$ to the optimal $g_s^{\rm F}dW_s$ does not influence the spectrums significantly in this example.
\begin{figure}[ht]
    \centering
    \includegraphics[width = 0.7 \linewidth]{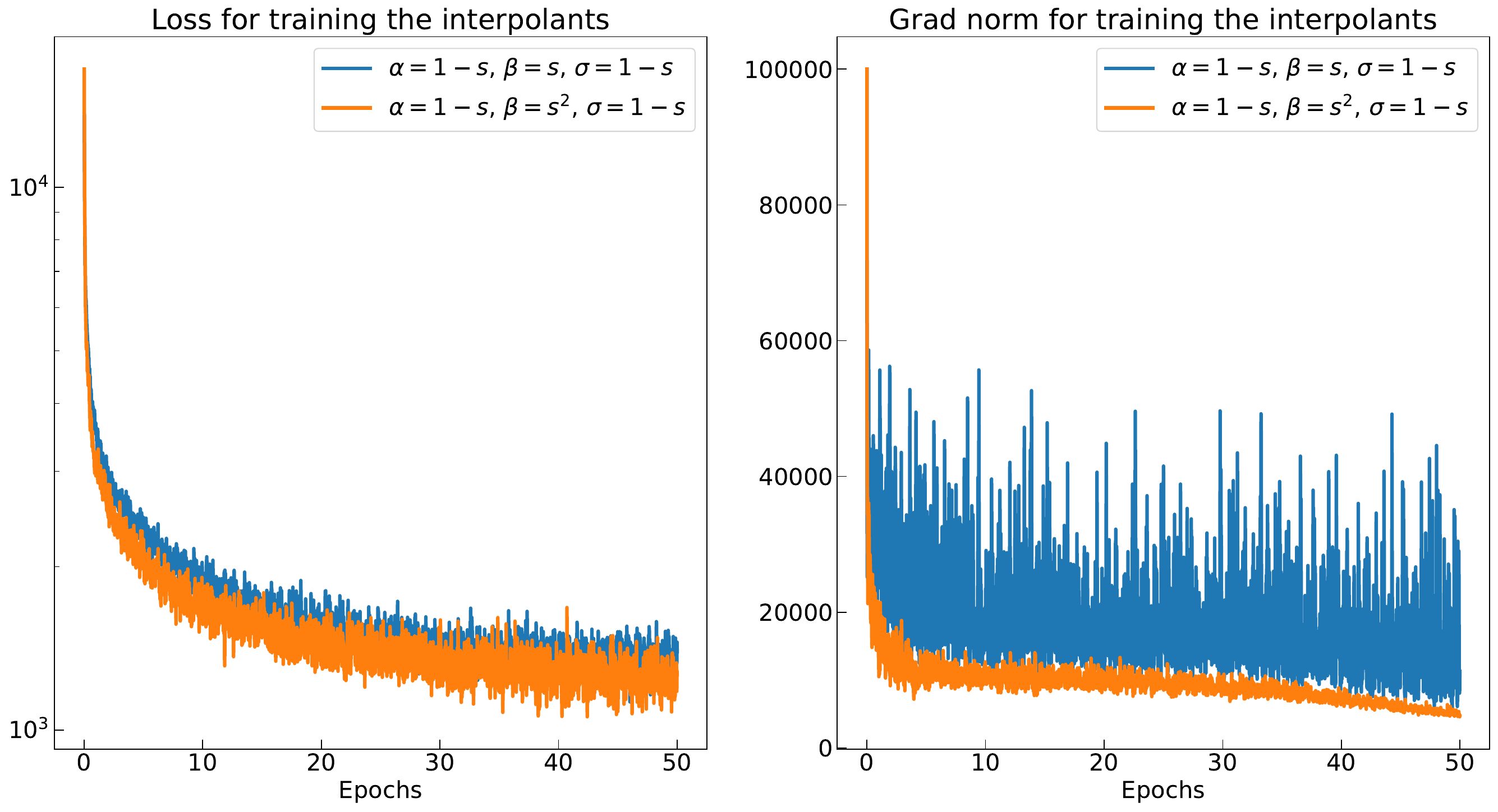}
    \vspace{1em}
    \caption{The training loss on the stochastic Navier-Stokes experiments and the norm of gradients of parameters during the training for stochastic interpolants with $\beta_s = s$ or $\beta_s = s^2$. We observe that the choice $\beta_s = s^2$ makes the gradient norm more stable.}
    \label{fig:loss_follmer_interpolants}
\end{figure}

\begin{figure}[t!]
    \centering
\begin{overpic}[width=0.24\linewidth]{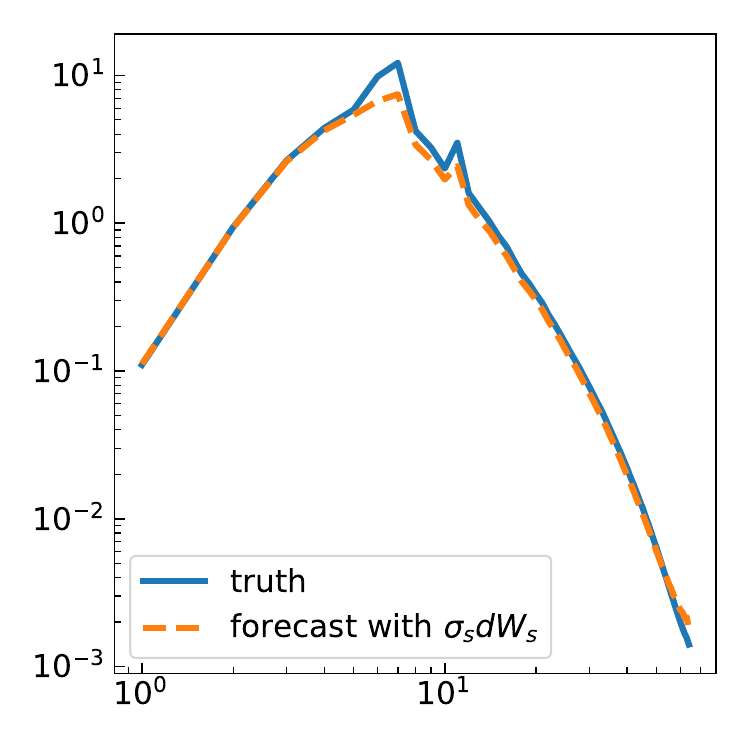}\put(23,-6){$\beta_s  = s$, use $\sigma_s dW_s$}
    \end{overpic}
    \begin{overpic}[width=0.24\linewidth]{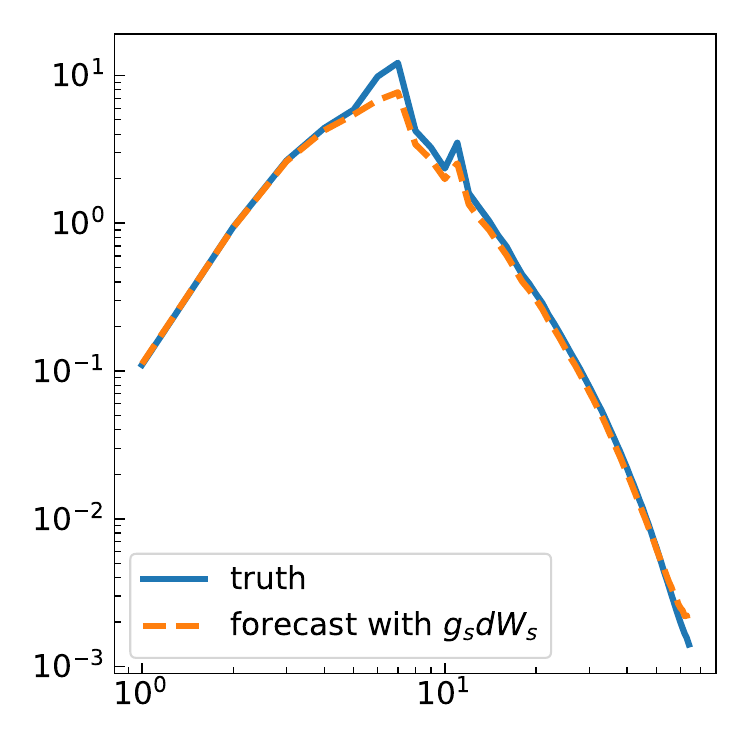}\put(23,-6){$\beta_s  = s$, use $g_s^{\rm F} dW_s$}
    \end{overpic}
    \begin{overpic}[width=0.24\linewidth]{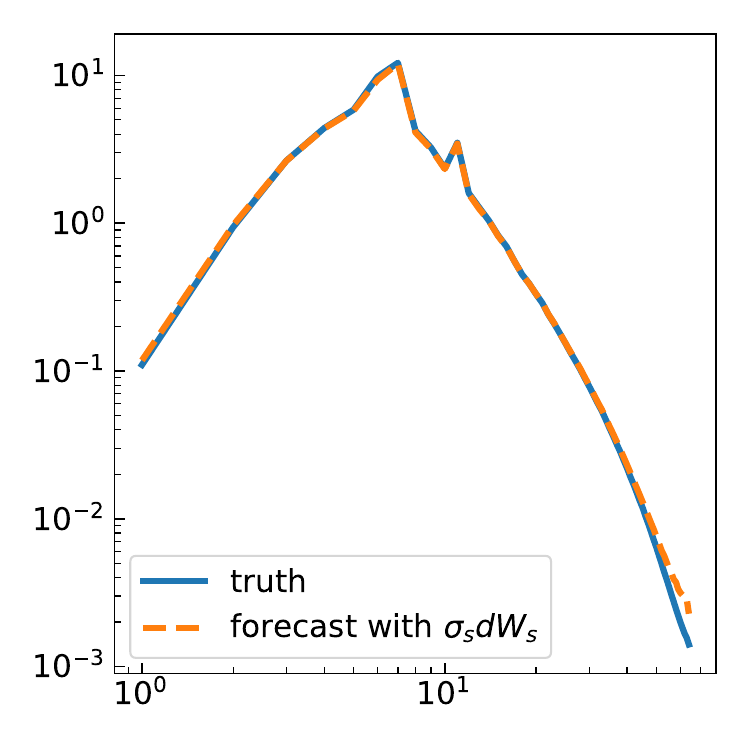}\put(23,-6){$\beta_s  = s^2$, use $\sigma_s dW_s$}
    \end{overpic}
    \begin{overpic}[width=0.24\linewidth]{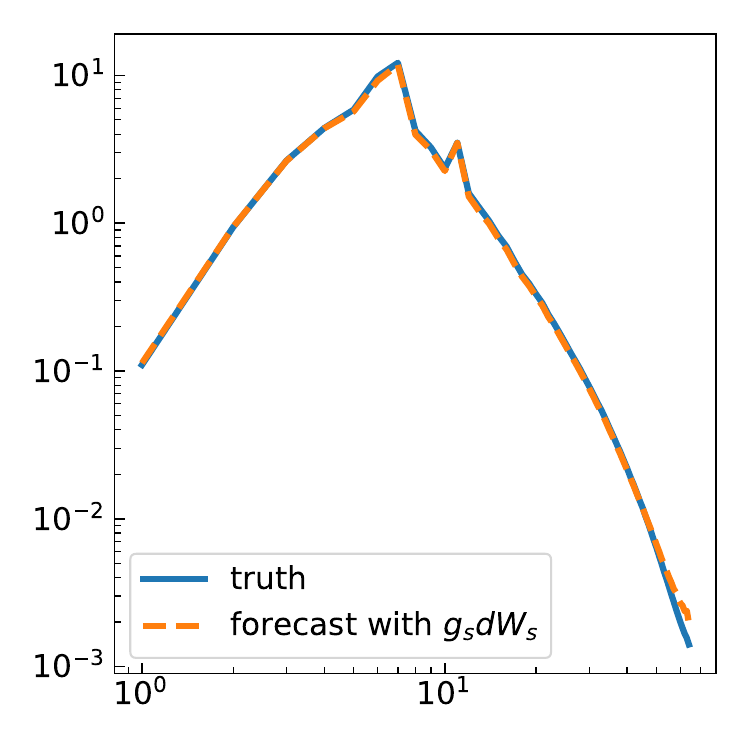}\put(23,-6){$\beta_s  = s^2$, use $g_s^{\rm F} dW_s$}
    \end{overpic}
    \vspace{1em}
    \caption{The enstrophy spectra of obtained samples from different generative models, compared to the truth, for interpolants $\alpha_s =1-s,\beta_s =s,\sigma_s =1-s$ and interpolants $\alpha_s =1-s,\beta_s =s^2,\sigma_s =1-s$. The generative SDE is chosen such that the diffusion term are $\sigma_s dW_s$ or $g_s^{\rm F}dW_s$. The spectrum is computed by averaging over $50$ different $\omega_t$ (independently drawn from the invariant distribution) and $300$ independent ensembles of $\omega_{t+\tau}$ (simulation data) or $\hat{\omega}_{t+\tau}$ (forecast results). All models are trained for lag $\tau=1$, with the same $2$M-parameter Unet, $200$K data, and the same training procedures as outlined in the beginning of this section. We observe that using $\beta_s = s^2$ leads to better performance in terms of the spectrums; this indicates that the stable gradient norm for $\beta_s = s^2$ (see Figure \ref{fig:loss_follmer_interpolants}) is important in the generation quality. Changing $g_s$ from $\sigma_s$ to the optimal $g_s^{\Fo}$ does not influence the spectrum significantly.}
    \label{fig:spectrum_interpolants}
\end{figure}
\paragraph{KL divergence comparisons between $\sigma_s dW_s$, $g_s^{\rm F}dW_s$, and Gaussian base ODE} We fix one initial vorticity field and compare the KL divergence between the true conditional distribution and the generated distribution. Here the KL divergence is calculated for the 1D conditional distributions of total enstrophy and energy, for a fixed $\omega_t$ and $\tau=1$. The goal here is to test whether changing the diffusion coefficient from $\sigma_s$ to $g_s^{\rm F} $ could improve the KL divergence accuracy of the generated distribution at $s=1$. Note that our theory in Theorem \ref{th:interp:f:si} shows that changing from $\sigma_s$ to $g_s^{\rm F}$ could improve the path KL divergence, which is an upper bound on the KL divergence of the generated marginal distribution at $s=1$. 

We consider the generative SDEs corresponding to the stochastic interpolant with $\alpha_s =1-s, \beta_s =s^2, \sigma_s =1-s$. We can vary the diffusion coefficient of the generative SDEs; we choose $\sigma_s  = 1-s$, or $g_s^{\rm F}  = \sqrt{(3-s)(1-s)}$ that corresponds to a F\"ollmer process. We also compare the results with the Gaussian base ODE generative model (equivalent to flow matching \cite{lipman2022flow, liu2022flow,albergo2022building}, corresponding to $\alpha_s =1-s, \beta_s =s$, and $x_0$ in the interpolant obeys an $\mathsf{N}(0,\Id)$ distribution). For this Gaussian base ODE, we use the same UNet architecture described in the beginning of Appendix \ref{sec:detail:nse}; again, the conditioning in the vector field is achieved by channel concatenation in the input. The quantitative KL results are reported in Table \ref{table: KL of different models} and the densities are shown in Figure \ref{fig:conditional density 1D slices}. We observe that SDE approaches lead to better KL accuracy. Moreover, changing the diffusion coefficients could potentially improve the KL accuracy, justifying the flexibility and usefulness of the interpolant approach.

\begin{table}[ht]
\centering
\begin{tabular}{ccc}
\toprule
{KL: truth versus generation} & density of total enstrophy & density of total energy \\ \midrule
SDE with $\sigma_s dW_s$ & 8.49e-3$\pm$1.57e-3    & \textbf{4.01e-3$\pm$8.95e-4}                                           \\
SDE with $g_s^{\rm F} dW_s$      & \textbf{2.79e-3$\pm$9.19e-4}             & 7.21e-3$\pm$1.58e-3                            \\
Gaussian base ODE        & 3.63e-3$\pm$9.63e-4            & 2.17e-2$\pm$2.50e-3                           \\ \bottomrule
\end{tabular}
\caption{Computed KL divergence between several 1D distributions of generated conditional samples and the true simulated samples. The 1D distributions are regarding total enstrophy and total energy of $\omega_{t+\tau}$, given a fixed $\omega_t$ and $\tau=1$. We first use kernel density estimation based on $6000$ ensembles to compute the density, and then compute the KL divergence with respect to the truth. Here we compare between the generated samples via SDEs with $\sigma_s dW_s$, via SDE with $g^{\rm F}_s dW_s$ which corresponds to a F\"ollmer process, and via ODEs with Gaussian base \cite{lipman2022flow, liu2022flow,albergo2022building}. Here to construct the SDEs, we use the interpolant $\alpha_s =1-s, \beta_s =s^2,\sigma_s =1-s$. In every column, the smallest value is highlighted in bold. The numbers are presented in the format \textbf{mean$\pm$std} where we do bootstrap to resample $6000$ ensembles with replacement and get different KL results and report the \textbf{mean and std}; this std provides useful information of the sensitivity of KL with respect to resampling. We observe SDE performs better in terms of KL.}
\label{table: KL of different models}
\end{table}

\clearpage
\begin{figure}[h]
    \centering
    \begin{overpic}[width=0.35\linewidth]{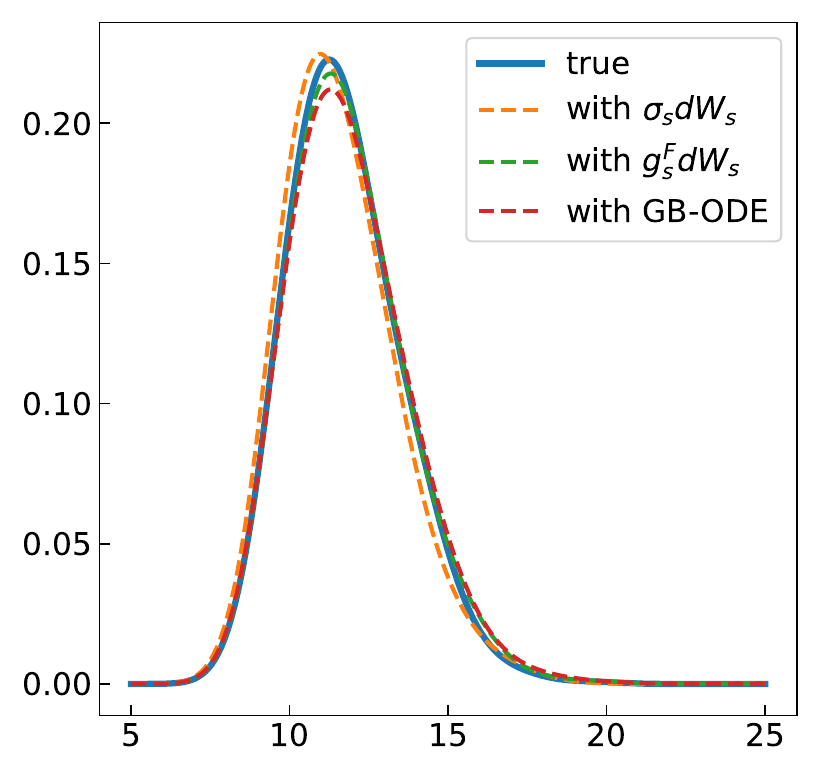}\put(4,-4){\begin{minipage}{\textwidth} density of total enstrophy of $\omega_{t+\tau}$ given $\omega_t$ \end{minipage}}
    \end{overpic}~~~~~~~~~~~~~~~~~~~~~~
    \begin{overpic}[width=0.35\linewidth]{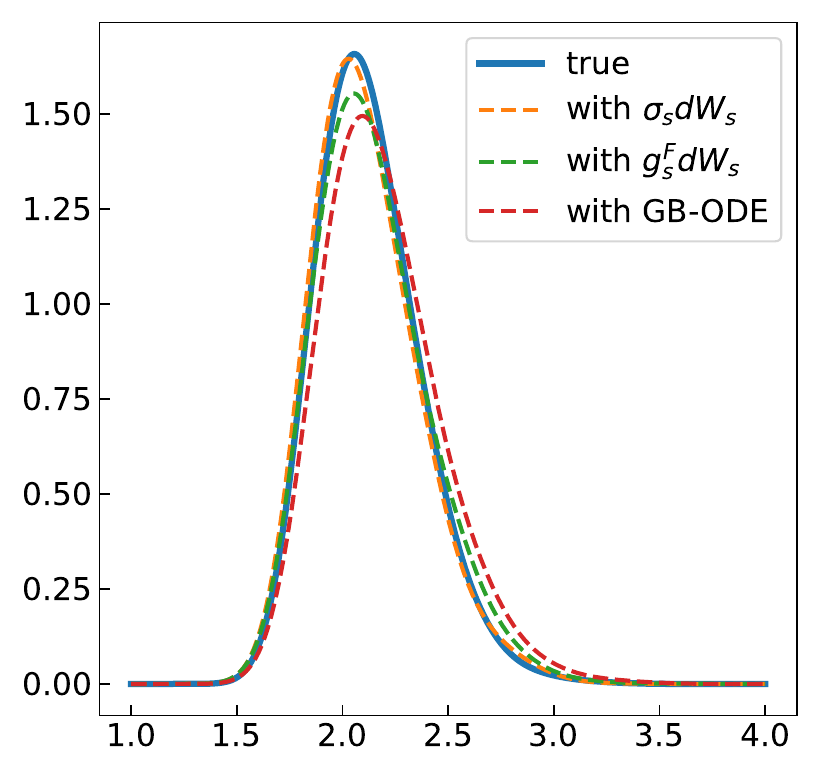}\put(4,-4){\begin{minipage}{\textwidth}density of total energy of $\omega_{t+\tau}$ given $\omega_t$\end{minipage}}
    \end{overpic}
    \vspace{2em}
    \caption{The 1D conditional distributions of total enstrophy and total energy of $\omega_{t+\tau}$, given a fixed initial vorticity field $\omega_t$ and $\tau=1$. Here we compare between the truth, generated samples via SDEs with $\sigma_s dW_s$, via SDE with $g^{\rm F}_s dW_s$ which corresponds to a F\"ollmer process, and via ODEs with Gaussian base \cite{lipman2022flow, liu2022flow,albergo2022building}. We use kernel density estimation based on $6000$ ensembles to compute the density. Here to construct the SDEs, we use the interpolant $\alpha_s =1-s, \beta_s =s^2,\sigma_s =1-s$. Quantitative KL divergence between these distributions with truth are reported in Table \ref{table: KL of different models}. We observe that SDEs lead to better density quality compared to ODEs. For the enstrophy, $g_s^{\rm F}dW_s$ behaves better while for the energy, $\sigma_s dW_s$ works better, for this fixed $\omega_t$. We leave as a future work to study the effect of the F\"ollmer $g_s^{\rm F}$ in more details.}
    \label{fig:conditional density 1D slices}
\end{figure}

\subsection{Video generation}
\label{app:video}

\begin{algorithm}[h]
  \caption{VideoEM: (Euler-Marayuma in Latent Space).}\label{alg:videoEM}
\begin{algorithmic}
\STATE{{\bfseries Input:}  Model $\hat b_s(y,y_0;y^t,t)$; previous frames $x^{1,\ldots, t-1}$; grid $s_0=0<s_1 \cdots< s_{N} =1$ with $N\in \N$; \textit{iid} $\eta_n \sim {\sf N}(0,\Id)$ for $n=0:N-1$.}
\STATE Set $\Delta s_n = s_{n+1}-s_n$ for $n=0:N-1$.
    \STATE{Set $Y_0 = \texttt{Encode}(x^{t-1})$}
    \FOR{$n=0:N-1$}
        \STATE{Draw $j \in \texttt{Unif}(2,\ldots,t-1)$}
        \STATE{$y^{t-j} = \texttt{Encode}(x^{t-j})$}
        \STATE{$Y_{n+1} = Y_n + \hat b_{s_n}(Y_n, Y_0; y^{t-j},t-j) \Delta s_n + \sigma_{s_n}  \sqrt{\Delta s_n} \eta_n$}
    \ENDFOR
    \STATE{$\hat{x}^t = \texttt{Decode}(Y_{N})$}
\STATE{\textbf{Return}: Generated frame $\hat{x}^t$}
\end{algorithmic}
\end{algorithm}
\paragraph{Training.} 
We train models for 250k gradient steps using AdamW starting at a learning rate of 2e-4. 
We use the UNet architecture popularized\footnote{\href{https://github.com/lucidrains/denoising-diffusion-pytorch/blob/main/denoising_diffusion_pytorch/classifier_free_guidance.py}{We use the lucidrains repository.}} in
\citet{ho2020denoising}.
We modify the architecture to condition on past frames by concatenating them along the channel dimension of the input. Each model is trained on four A100 GPUs for approximately 1-2 days.

\end{document}